\documentclass[preprint,12pt,authoryear,a4paper]{elsarticle}

\setcitestyle{numbers,square}
\usepackage{amssymb}
\usepackage{comment}
\usepackage{amsmath}
\usepackage{amsthm}
\usepackage{stmaryrd}
\usepackage{multirow}
\usepackage{booktabs}
\usepackage{tikz}
\usepackage{listings}
\usepackage{hyperref} 
\usepackage[ruled,linesnumbered]{algorithm2e}
\usepackage{soul}

\newtheorem{definition}{Definition}
\newtheorem{example}{Example}

\newtheorem{proposition}{Proposition}
\usepackage{graphicx}
\hypersetup{colorlinks = true, allcolors = blue}

\newcommand{\xy}[1]{{\color{black}{#1}}} 
\newcommand{\zmw}[1]{{\color{black}{#1}}} 
\newcommand{\zm}[1]{{\color{black}{#1}}}
\newcommand{\alt}[1]{{\color{black}{#1}}}
\journal{JLAMP}

\begin{document}
\sloppy
\begin{frontmatter}

\title{Weighted Automata Extraction and Explanation of
Recurrent Neural Networks for Natural Language Tasks}

\author[address1]{Zeming Wei}
\author[address2]{Xiyue Zhang}
\author[address1]{Yihao Zhang}
\author[address1]{Meng Sun\corref{correspondingauthor}}
\cortext[correspondingauthor]{Corresponding author: sunm@pku.edu.cn}

\address[address1]{School of Mathematical Sciences, Peking University, Beijing, China}
\address[address2]{Department of Computer Science, University of Oxford, Oxford, United Kingdom}

% \author{Zeming Wei}
% \author{Xiyue Zhang}
% \author{Yihao Zhang}
% \author{Meng Sun\footnote{Corresponding author. Email: sunm@pku.edu.cn}}

% \address{School of Mathematical Sciences, Peking University, China}

\begin{abstract}
Recurrent Neural Networks (RNNs) have achieved tremendous success in processing sequential data,
yet understanding and analyzing their behaviours remains a significant challenge.
To this end,  many efforts have been made to extract finite  automata from RNNs, which are more amenable for analysis and explanation.
However, existing approaches like exact learning and compositional approaches for model extraction  have limitations in either scalability or precision.
\alt{In this paper, we propose a 
% \st{general} framework 
novel framework
of Weighted Finite Automata (WFA) extraction and explanation to tackle the limitations 
% for RNNs 
% \alt{customized} 
for 
% \xy{in the context of} 
natural language tasks.
% \xyc{cant be both general and customized??}
% \st{For the extraction part,} 
First, to address the transition sparsity and context loss problems we identified in WFA extraction for natural language tasks,
% Then 
we propose an empirical method to 
complement missing rules in the transition diagram,
and adjust transition matrices to enhance the context-awareness 
% ability 
of the WFA.
% \st{Additionally,} 
We also propose two data augmentation tactics to track more dynamic behaviours of RNN, \xy{which further allows us to improve the extraction precision}.
% Based on our WFA extraction approach, 
Based on the extracted model, we 
% further 
% \xy{show how the extracted automata assist in the explanation of RNNs.}
propose an explanation method for RNNs \xy{including a word embedding method -- Transition Matrix Embeddings (TME) and TME-based task oriented explanation for the target RNN.}}
% Specifically, 
% % by leveraging the extracted WFA, 
% % we introduce a word embedding method named Transition Matrix Embeddings (TME), and construct task-oriented explanations of the target RNN, including 
% a word-wise global explanation of RNNs and 
% a contrastive method to explain task-oriented word semantics learned by RNNs.}
% Experiment results on two popular natural language datasets 
\xy{Our evaluation demonstrates the advantage of our method
%\footnote{Code is available at \url{https://github.com/weizeming/Extract\_WFA\_from\_RNN\_for\_NL}.}
in extraction precision than existing approaches, and 
% We also validate 
the effectiveness of TME-based explanation method in applications to pretraining and adversarial example generation.}
\end{abstract}

\begin{keyword}
Abstraction \sep Explanation \sep Weighted Finite Automata  \sep Natural Languages \sep Recurrent Neural Networks 

\end{keyword}

\end{frontmatter}
\section{Introduction}
In the last decade, deep learning (DL) has been widely deployed in a range of applications, such as image processing~\cite{He_2016_CVPR}, speech recognition~\cite{abdel2014convolutional} and natural language processing~\cite{goldberg2017neural}.
In particular, recurrent neural networks (RNNs) achieve great success in sequential data processing, e.g., time series forecasting~\cite{che2018recurrent}, text classification~\cite{wang2019convolutional} and language translation~\cite{datta2020neural}.
However, the complex internal design and gate control of RNNs make the interpretation and analysis of their behaviours rather challenging.
% \alt{\st{There are two main challenges to understanding RNNs, the first is how to abstract RNNs as explainable models, and the second is how to use these models to interpret RNNs.}}

% \alt{For the first part,}
Recently,
much progress has been made to abstract RNN as a finite automaton, that is,  a finite-state model with explicit states and transition matrix to characterize the behaviours of RNN in processing sequential data.
%The extracted automaton also provides a practical foundation for analyzing and interpreting RNN behaviours, based on which existing mature techniques, such as logical formalism~\cite{logic} and model checking~\cite{modelChecking}, can be leveraged for RNN analysis
Up to the present, a series of extraction approaches leverage explicit learning algorithms (e.g., $L^*$ algorithm~\cite{angluin}) to extract a surrogate model of RNN.
Such exact learning procedure has achieved great success in capturing the state dynamics of RNNs when processing formal languages~\cite{weiss2018,weiss2019,ok2020}.
However, the computational complexity of the exact learning algorithm limits its scalability to construct abstract models from RNNs for natural language tasks.
Another technical line for automata extraction 
from RNNs 
is the compositional approach, which uses unsupervised learning algorithms to obtain discrete partitions of RNNs' state vectors and construct the transition diagram based on 
% the discrete clusters and 
the concrete state dynamics of RNNs.
\alt{This approach demonstrates better scalability and has been applied to robustness analysis and repairment of RNNs on large-scale tasks~\cite{wang2018,wang2018verification,du2019,dong2020,du2020,xie2021},
but falls short in extraction precision.
% as a trade-off to the computational complexity
% \alt{, but is faced with the problem of extraction consistency as a trade-off to the computational complexity.}
%, the compositional approach is faced with the problem of extraction consistency.
% \alt{Nevertheless, \st{since RNNs can be applied to various tasks involving sequential data,} }and customized explanation 
A precise and scalable extraction approach for RNNs in the context of natural language tasks is needed.} 
% has not been well explored.}

\alt{Regarding \xy{model-based explanation}, current extraction methods \xy{are limited to utilizing finite automata as a global interpretable model with explicit states and transition rules for RNNs}. 
% However, 
\xy{The information extracted in the 
% \st{complex} 
transition diagram of automata is not fully exploited in understanding RNN behaviors for natural language tasks.
In particular, given that the alphabet size of natural language datasets is quite large, the extracted rules in the transition matrix are difficult to grasp and interpret.
% \st{more explicit and} 
A more comprehensible explanation method that can effectively exploits the extracted information to assist in understanding RNN behaviors remains underexplored.
}}
% \st{for natural language tasks}

\alt{
In this paper, 
we propose a general framework
of Weighted Finite Automata (WFA) extraction and explanation for RNNs to tackle the above challenges.
% we aim to improve the automata-based extraction and explanation of RNNs designed for natural language tasks.
To address the first challenge, we 
% identify
% We first identify several problems that occur in this context, and 
propose a complete pipeline to extract more precise automata for RNNs in the context of natural language tasks. We identify two problems that cause precision deficiency in natural language tasks: (1) \textit{transition sparsity}: the transition dynamics are usually sparse in natural language tasks, due to the large alphabet size 
% (than formal languages) 
and the dependency on a finite set of (sequential) data in the extraction procedure.
% when processing low-frequency tokens (words),
% By exploring the problems specified for natural language tasks,
% we find that two issues of automata extraction are yet to be addressed. Firstly,
% the alphabet size of natural language datasets is far larger than formal languages; meanwhile, the extraction procedure is based on a finite set of (sequential) data. As a result, the transition dynamics are usually scarce when processing low-frequency tokens (words),
% and the sparsity problem in the transition matrices severely impacts the behavior consistency between RNNs and the extracted models.
% \alt{We refer this issue as \textbf{transition sparsity} problem.}
%However, the transition sparsity of the extracted automata for natural language tasks is yet to be addressed.
(2)\textit{context loss:} the tracking of long-term context of RNNs (e.g., LSTM networks~\cite{lstm}) is inevitably compromised due to the abstraction.
To deal with the transition sparsity problem, 
% \xy{especially when} 
% that 
% no transition rules are learned at a certain state for a certain word (also referred to as \textit{missing rows}),
% in transition matrices,
we propose a method to fill in 
% the transition rules for 
the missing transition rules based on the semantics of abstract states.
We also propose two tactics to augment the data samples to learn more transition behaviours of RNNs, which further alleviates the transition sparsity problem.
To enhance the context awareness of WFAs, 
we adjust the transition matrices to preserve partial context information from the previous states.
}

\alt{
% As discussed, the extracted WFAs are not sufficient to interpret RNNs for natural language tasks.
% As the extended version of~\cite{wei2022extracting}, in this paper
To address the second challenge, 
we utilize the extracted WFAs to interpret the behaviours of RNNs.
Motivated by the observation that
%Notice that 
the transition matrices of the extracted WFAs capture the behaviour of the source RNNs,
% when processing different input words $\sigma$.
%Motivated by this observation, 
we propose a word embedding method -- Transition Matrix Embeddings (TME) 
% by leveraging the extracted information in WFAs 
to construct task-oriented explanations for the target RNNs.
% {Overall, the proposed embedding method TME captures the semantics of the words that the RNNs learned from the task. 
% We then 
Further, by leveraging the information captured in TME, we propose a global explanation method for word attribution to RNNs' decisions
   %Leveraging 
   % to analyze the influence of each word 
   % Further, we propose 
    and a contrastive method to investigate the difference between task-oriented TME and pretrained word embeddings (\textit{e.g.}, Glove~\cite{pennington2014glove}). We 
    % empirically show the difference of task-oriented word semantics that the RNN learned from the conventional semantics,
    % and 
    validate the effectiveness of the contrastive explanation with applications to pretraining boost and adversarial example generation\footnote{Code is available at \url{https://github.com/weizeming/Extract\_WFA\_from\_RNN\_for\_NL}}.}
    
    % \alt{In a word, our contribution can be divided into two parts. The former is a WFA extraction algorithm from RNNs for natural language tasks, which can be summarizes as:}
    \xy{We summarize our contributions as follows:
    \begin{enumerate}
        \item[(a)] 
        We propose 
        a complete WFA extraction algorithm from RNNs designed for natural language tasks.
        % \alt{We identify the transition sparsity and context-aware problem in WFA extraction from RNNs for natural language tasks;}
        %A novel approach to extracting transition rules of WFAs from RNNs to address the transition sparsity problem;
        \item[(b)] \zmw{Experiments on benchmark datasets demonstrate that the proposed heuristic methods effectively improve the extraction precision by alleviating the transition sparsity and context loss problems.}
        % \alt{To address these problems, we propose a complete WFA extraction algorithm from RNNs customized for natural language tasks;}
        % %An heuristic method of adjusting transition rules to enhance the context-aware ability of WFAs;
        \item[(c)] We propose a novel word embedding -- Transition Matrix Embeddings (TME), based on which a global explanation method for word attribution and a contrastive approach for task-oriented explanation 
        % semantics of words the 
        of RNNs are proposed. 
        % learned.
        % \alt{Experiment on two benchmark datasets shows the effectiveness of our approach on extraction precision.}
        %A data augmentation method on training samples to track more transition behaviours of RNNs.}
    \end{enumerate}}

    % % \alt{The latter part of our contribution is an WFA based explanation framework of the RNNs for natural langauge tasks, which can be summarizes as:}
    % \begin{enumerate}
    %     % \item[(d)] \alt{We propose a word embedding method named Transition Matrix Embeddings (TME) by leveraging the extracted WFA;}% to explain and interpret behaviours of the source RNNs;}
    %     \item[(e)] \alt{Based on TME, we further propose a global explanation approach of RNNs by analyzing the influence of each word on its decision;}% based on TME;}
    %     \item[(f)] \alt{We finally propose a contrastive approach to understand task-oriented semantics of words the RNNs learned.}
    % \end{enumerate}
    
    \alt{The organization of this paper is as follows. 
    In Section \ref{sec:preliminary}, we present preliminaries about recurrent neural networks, weighted finite automata, and related notations and concepts.
    In Section \ref{sec:approach}, we present our transition rule extraction approach, including 
   an overview
    % a generic outline 
    on the automata extraction procedure, the transition rule complement method for transition sparsity, the transition rule adjustment method for context-awareness enhancement,
    %We then present 
    and the data augmentation tactics.
    %in Section \ref{sec:augmentation} to reinforce the learning of dynamic behaviours from RNNs, 
    %along with the computational complexity analysis of the overall extraction approach.
    In Section \ref{sec:experiments}, we present the experimental evaluation towards the extraction consistency of our approach on two natural language tasks. 
    % Additionally, 
    We introduce the transition matrix embedding based explanation framework for RNNs in Section~\ref{sec:expla}, and discuss our extraction algorithm including computational complexity analysis and applicability to other RNNs at the end of Section~\ref{sec:expla}.}
    Finally, we discuss related works in Section \ref{sec:related} and conclude our work in Section \ref{sec:conclusion}.
    
    %\xyc{Shall we put explanation section before experiments?}
    % 

\section{Preliminaries}
\label{sec:preliminary}
In this section,
we present the notations and definitions that will be used throughout the paper.
% We appoint the notations and basic concepts used in this paper.
% \paragraph{Notations.}
Given a finite alphabet $\Sigma$, we 
use $\Sigma^*$ to denote the set of sequences over $\Sigma$ and $\varepsilon$ to denote the empty sequence.
% denote $\Sigma^*$ as the set of sequences over $\Sigma$ and $\varepsilon$ as the empty sequence.
For $w\in\Sigma^*$, we use $|w|$ to denote its length, its $i$-th word as $w_i$ and its prefix with length $i$ as $w[:i]$.
For $x\in\Sigma$, $w\cdot x$ represents the concatenation of $w$ and $x$. 
% For $w\in\Sigma^*$, we denote $|w|$ as its length, $w_i$ as its $i$-th word and $w[:i]$ as its prefix with length $i$.
% For $x\in\Sigma$, $w\cdot x$ presents the concatenation of $w$ and $x$. 

\begin{definition}[RNN]
A \textit{Recurrent Neural Network (RNN)} for natural languages is a tuple
$\mathcal{R}=(\mathcal X, \mathcal{S}, \mathcal{O},f, p)$,
where $\mathcal X$ is the input 
space; 
% specifically is an 
% alphabet for NLP tasks;
$\mathcal{S}$ is the internal state space;
$\mathcal{O}$ is the probabilistic output space;
$f: \mathcal{S}\times\mathcal X\to \mathcal{S}$ is the transition function;
$p: \mathcal{S}\to\mathcal{O}$ is the prediction function.
\end{definition}
%\xy{zm: seems confusion with the alphabet and input space; fix it.}
%\zm{replace $\Sigma$ with $\mathcal X$}
%\xy{Paragraphs: replicated dots for the paragraph format ..}

\paragraph{RNN Configuration}
In this paper, we consider RNN as a black-box model and focus on its stepwise probabilistic output for each input sequence.
% The following definition of configuration maps each sentence in $\Sigma^*$ to a probabilistic output given by RNN.
The following definition of configuration characterizes the probabilistic outputs in response to a sequential input
% in $\Sigma^*$ to a 
fed to RNN.
    % First 
    Given an alphabet $\Sigma$,
    % as input space,
    let $\xi:\Sigma\to\mathcal X$ be the function that maps each word in $\Sigma$
    to its embedding vector in $\mathcal X$.
    We define $f^*: \mathcal{S}\times\Sigma^*\to\mathcal{S}$ recursively as $f^*(s_0,\xi(w\cdot x))=f(f^*(s_0, \xi(w)),\xi(x))$ and $f^*(s_0,\varepsilon)=s_0$,
    where $s_0$ is the initial state of $\mathcal{R}$.
The RNN configuration $\delta:\Sigma^*\to \mathcal{O}$ is defined as $\delta(w)=p(f^*(s_0, w))$.
%\xy{zm: errors in the f definition and delta definition, revise.}\zm{done}
% \xy{alphabet as input space?}
% \zm{maybe word embedding vector is better}

\paragraph{Output Trace}
To record the stepwise behavior of RNN when processing an input sequence $w$, 
% in $\Sigma^*$,
we define the \textit{Output Trace} of $w$, i.e., the probabilistic output sequence, as
% of $w\in\Sigma^*$ as 
% a probabilistic output sequence
$T(w)=\{\delta(w[:i])\}_{i=1}^{|w|}$.
The $i$-th item of $T(w)$ indicates the probabilistic output given by $\mathcal{R}$ 
% when 
after taking the prefix of $w$ with length $i$ as input.

% So far we have complete the definitions about RNN.
% Next, we focus on the preliminaries for weighted finite automaton(WFA).
\begin{definition}[WFA]
Given a finite alphabet $\Sigma$, a \textit{Weighted Finite Automaton (WFA)} over $\Sigma$ is a tuple
$\mathcal{A}=(\hat S, \Sigma, E, \hat s _0, I, F)$, where $\hat S$ is the finite 
set of abstract states;
% abstract state set;
$E=\{E_\sigma|\sigma\in \Sigma\}$ is the
set of transition matrix $E_\sigma$ with size $|\hat S|\times|\hat S|$ for each token $\sigma\in\Sigma$;
% the transition matrices set, for $\sigma\in\Sigma$, 
% $E_\sigma$ denotes its transition matrix 
% with size $|\hat S|\times|\hat S|$;
$\hat s _0\in \hat S$ is the initial state;
$I$ is the initial vector, a row vector with size $|\hat S|$;
$F$ is the final vector, a column vector with size $|\hat S|$.
\end{definition}

\paragraph{Abstract States}
Given a RNN $\mathcal{R}$ and a dataset $\mathcal D$,
let $\hat{\mathcal O}$ denote all stepwise probabilistic outputs given by executing $\mathcal R$ on $\mathcal D$,
i.e. $\hat{\mathcal{O}}=\bigcup\limits_{w\in\mathcal D}T(w)$.
The abstraction function $\lambda:\hat{\mathcal{O}}\to \hat{S}$ maps each probabilistic output to an abstract state $\hat s\in \hat S$.
As a result, the output set is divided into a number of abstract states by $\lambda$.
% where $\hat S$ is a finite set of abstract states.
For each $\hat s\in\hat S$, the state $\hat s$ has explicit semantics that
% which means 
the probabilistic outputs corresponding to $\hat s$ has similar distribution.
% Hence, the output set is divided into a number of abstract states by $\lambda$.
In this paper, we leverage the \textit{k-means} algorithm to construct the abstraction function.
We cluster all probabilistic outputs in \alt{$\hat{\mathcal{O}} $} into some abstract states.
In this way, we construct the set of abstract states $\hat S$ with these discrete clusters and an initial state $\hat s_0$.

For a state $\hat s \in \hat S$, 
we define the \textit{center} of $\hat s$ as the 
average value of the probabilistic outputs \alt{$\hat o\in \hat{\mathcal{O}}$ }which are mapped to $\hat s$. 
More formally, the center of $\hat s$ is defined as follows:
$$\rho(\hat s)=\underset{\lambda({\hat o}) = \hat s}{\mathrm{Avg}}\{\hat o\}.$$
The center $\rho(\hat s)$ represents an approximation of the distribution tendency of probabilistic outputs $\hat o$ in $\hat s$.
% The center $\rho(\hat s)$ presents an approximation for each probabilistic output $\hat o$ consisting with $\hat s$.
% It needs to be mentioned that we also select $\rho(\hat s_i)$ as the $i$-th item of terminal vector $P$,
% indicating that if the WFA terminates at state $\hat s_i$, we will take $\rho(\hat s_i)$ as its weight.
\xy{We then use the center $\rho(\hat s)$ as its weight for each state $\hat s \in \hat S$.
% as the center captures an approximation of the distribution tendency of this state.
% Therefore, 
The final vector $F$ is thus formulated as $(\rho(\hat s_0),\rho(\hat s_1),\cdots, \rho(\hat s_{|\hat S|-1}))^t$.}
%\xy{zm:better revise final vector P to F.}
%\zm{done}

\paragraph{Abstract Transitions} In order to capture the dynamic behavior of RNN $\mathcal{R}$,
% when the current output is $y\in \mathcal{O}$ and the next input word is $\sigma$,
we define the abstract transition as a triple $(\hat{s}, \sigma, \hat{s}')$ where the original state $\hat{s}$
is the abstract state corresponding to a specific output $y$, i.e. $\hat s=\lambda(y)$; $\sigma$ is the next word of the input sequence to consume; $\hat s '$
is the destination state $\lambda(y')$ after $\mathcal{R}$ reads $\sigma$ and outputs $y'$. 
% Suppose that the new probabilistic output of $\mathcal R$ after it read $\sigma$ is $y'$,
% then $\hat s '$ can be formulated as $\hat s ' = \lambda(y')$.
% Particularly, if $\sigma$ is the first word in its sentence, then $\hat s$ should be the initial state $\hat s_0$.
%\xy{zm: if there are multiple initial states, they  cannot be simply represented as one state $s_0$.}
%\zm{remove the previous sentence}
We use $\mathcal T$ to denote the set of all abstract transitions tracked from the execution of $\mathcal R$ on training samples.

\paragraph{Abstract Transition Count Matrices} For each word $\sigma\in \Sigma$,
the abstract transition count matrix
% (abbrev. count matrix) 
of $\sigma$ is a matrix $\hat T_\sigma$
with size $|\hat S|\times|\hat S|$. 
The count matrices record the number of times that each abstract transition is triggered.
%\xy{revise this sentence.}. \zm{done}
Given the set of abstract transitions $\mathcal{T}$, the count matrix of $\sigma$ can be calculated as
$$\hat T_\sigma[i,j]=\mathcal{T}.count((\hat s_i,\sigma,\hat s_j)),\quad 1\le i,j\le |\hat S|.$$ 

% \begin{definition}[WFA]
% Given a finite alphabet $\Sigma$, a \textit{Weighted Finite Automaton (WFA)} over $\Sigma$ is a tuple
% $\mathcal{A}=(\hat S, \Sigma, E, \hat s _0, I, P)$, where $\hat S$ is the finite abstract state set;
% $E=\{E_\sigma|\sigma\in \Sigma\}$ is the transition matrices set, for $\sigma\in\Sigma$, $E_\sigma$ denotes its transition matrix with size $|\hat S|\times|\hat S|$;
% $\hat s _0\in \hat S$ is the initial state;
% $I$ is the initial vector, a row vector with size $|\hat S|$;
% $P$ is the final vector, a column vector with size $|\hat S|$.
% \end{definition}

% Technically, the states set $\hat S$ contains the abstract states and initial state $\hat s_0$,
As for the remaining components,
the alphabet $\Sigma$ is consistent with the alphabet of training set $\mathcal D$.
% the construction of transition matrices $E$ is elaborated in next section.
The initial vector $I$ is formulated according to the initial state $\hat s_0$.
%\xy{revise accordingly with regard to the alphabet and initial state.}
% For each abstract state in $\hat S$, we use its center as its weight,
% since the center presents an approximation of the probabilistic outputs consisting with this state.
% So the final vector $P$ is $(\rho(\hat s_0),\rho(\hat s_1),\cdots, \rho(\hat s_{|\hat S|-1}))^t$.

For an input sequence $w=w_1w_2\cdots w_n\in\Sigma^*$, the WFA will calculate
% give 
its weight following $$I\cdot E_{w_1}\cdot E_{w_2}\cdots E_{w_n}\cdot F.$$

% The detailed extracting workflow is elaborated in Section 3.

\begin{figure}[t]
    \includegraphics[width=\textwidth]{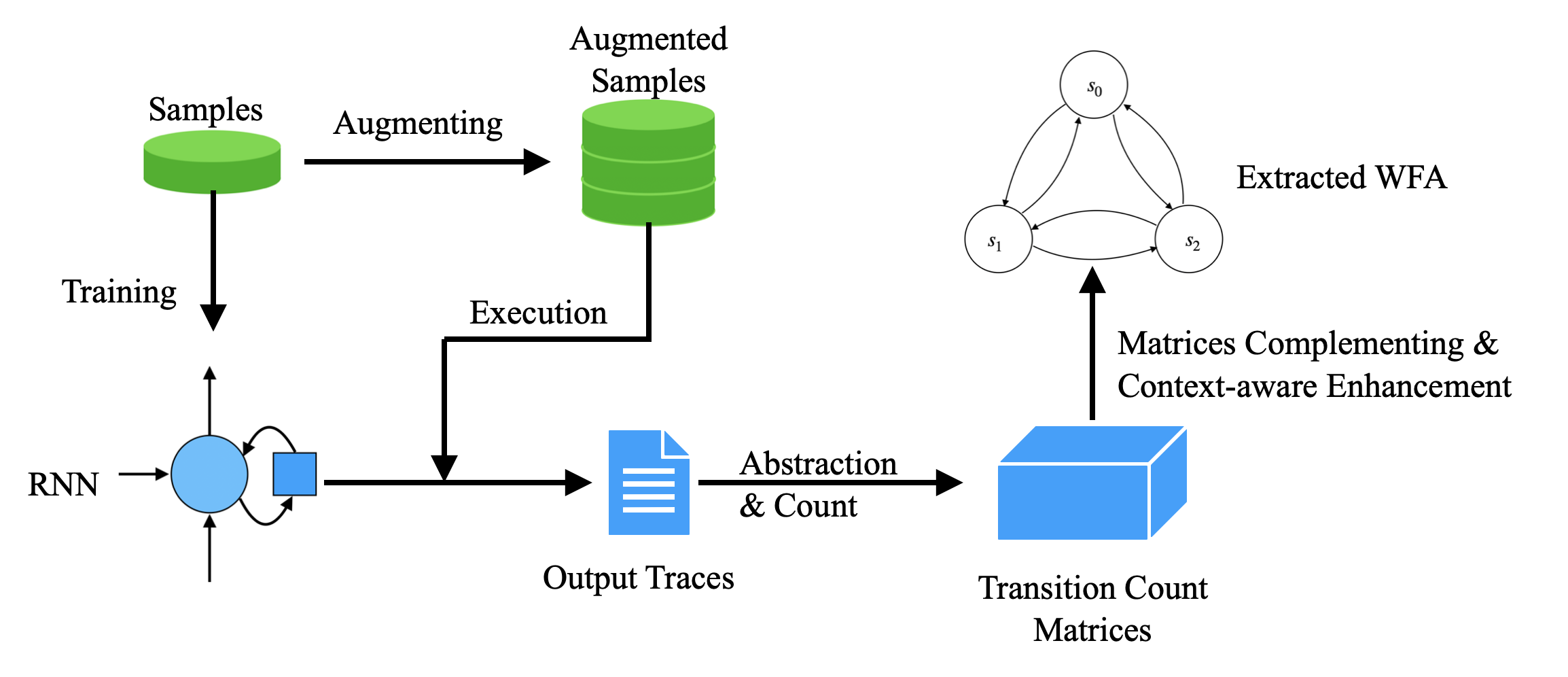}
    \caption{An illustration of our approach to extracting WFA from RNN.} \label{fig1}
\end{figure}

\section{Weighted Automata Extraction Scheme}
\label{sec:approach}
\subsection{Overview}
\label{subsec:outline}

We present the workflow of our extraction procedure
% is shown 
in Figure~\ref{fig1}.
% To start with,
As the first step,
we generate augmented sample set ${\mathcal D}$ from the original training set $\mathcal D_0$ to enrich the transition dynamics of RNN behaviours and alleviate the transition sparsity.
% to track more RNN behaviours.
%We propose two tactics, \textit{Synonym Replacement} and \textit{Dropout}, to achieve this, which is detailed in Section~\ref{sec:augmentation}.
% implement this.
% By generating more sentences, we can capture more behaviours of RNN,
% which provides more useful transition information to construct the transition matrices and mitigates the overfitting problem.
% The detailed augmentation tactics are discussed in Section 4.
Then, we execute RNN $\mathcal{R}$ on the augmented sample set ${\mathcal D}$,
% while 
and record the probabilistic output trace $T(w)$ of each input sentence $w\in \mathcal D$.
With the output set $\hat O=\bigcup\limits_{w\in \mathcal D}T(w)$, we cluster the probabilistic outputs into abstract states $\hat S$,
and generate abstract transitions $\mathcal T$ from the output traces $\{T(w)|w\in \mathcal D\}$.
All transitions constitute the abstract transition count matrices $\hat T_\sigma$ for all $\sigma\in\Sigma$.

%\xy{Change all ``base on'' to ``based on''. From here. All previous ones have been corrected.}
Next, we construct the transition matrices $E=\{E_\sigma|\sigma\in\Sigma\}$.
    Based on the abstract states $\hat S$ and count matrices $\hat T$,
    we construct the transition matrix $E_\sigma$ for each word $\sigma\in\Sigma$.
    % empirically.
    Specifically, we use frequencies to calculate the transition probabilities.
    Suppose that there are $n$ abstract states in $\hat S$.
    The $i$-th row of $E_\sigma$, which indicates the  probabilistic transition distribution over states when $\mathcal{R}$ is in state $\hat s _i$
    and 
    % reads 
    consumes $\sigma$, is calculated as 
    \begin{equation}
        E_\sigma[i,j] = \frac{\hat T_\sigma[i,j]}{\sum\limits_{k=1}^n \hat T_\sigma[i,k]}.\label{transition matrix}
    \end{equation}
    This empirical rule faces the problem that the denominator of (\ref{transition matrix}) could be zero,
    which means that the word $\sigma$ never appears when the RNN $\mathcal{R}$ is in abstract state $\hat s_i$.
    In this case, one should decide how to fill 
    in the transition rule of the \textit{missing rows} in $E_\sigma$.
    % tthe $i$-th row in $E_\sigma$.
    % We define this kind of rows as \textit{missing rows}.
    In Section \ref{subsec:missing_rows},
    we present a novel approach for transition rule complement.
    \alt{Further, 
    % for the sake of 
    to preserve more contextual information 
    % when executing WFA 
    % in the case of transition sparsity,
    during processing the input sequence, 
    % \xyc{Not because of sparsity right? use the motivation for context enhancement}
    we propose an approach to enhancing the context-awareness
    % ability 
    of WFA by adjusting the transition matrices,
    % The context-awareness enhancement approach 
    which is presented in Section \ref{subsec:context}.}
%{Note that our approach is generic and could be applied to other RNNs besides the domain of natural language processing.}

\subsection{Missing Rows Complement}
\label{subsec:missing_rows}
    %As discussed before, 
    Existing approaches for transition rule extraction usually face the problem of transition sparsity,
    i.e., \textit{missing rows} in the transition diagram.
    % the matrices construction rule faces the problem of processing \textit{missing rows}.
    % The rules for dealing with formal languages doesn't need to pay attention to this problem,
    % since the size of the alphabet is small and each word in the alphabet can appear sufficient number of times,
    % making the probability of occurrence of missing rows very limited.
    \xy{In the context of formal languages, the probability of the occurrence of missing rows is quite low, 
    % limited,
    % The rules for dealing with formal languages doesn't need to pay attention to this problem,
    since the size of the alphabet is small and each token in the alphabet can appear sufficient number of times.}
    % making the probability of occurrence of missing rows very limited.
    However, in the context of natural language processing, the occurrence of missing rows is quite frequent.
    The following proposition gives an approximation of 
    the occurrence frequency
    % times 
    of missing rows.

    \begin{proposition}
        % Given 
        Assume 
        an alphabet $\Sigma$ with $m=|\Sigma|$ words, a natural language dataset $\mathcal D$ over $\Sigma$ which has $N$ words in total,
        % with 
        a RNN $\mathcal{R}$ trained on $\mathcal D$,
        % it 
        % and 
        % abstracted as 
        the extracted abstract states $\hat S$ and 
        transitions 
        % transitional 
        $\mathcal{T}$.
        Let $\sigma_i$ 
        denote 
        the $i$-th most frequent word 
        occurred
        in $\mathcal D$ and $t_i = \mathcal{T}.count((*,\sigma_i,*))$
        indicates
        the occurrence times of $\sigma_i$ in $\mathcal D$.
        The median of $\{t_i|1\le i\le m\}$ can be estimated as $$t_{[\frac{m}{2}]}=\frac{2N}{m\cdot \ln m}.$$
    \end{proposition}

    \begin{proof}
        The Zipf's law~\cite{powers1998applications} shows that $$\frac{t_i}{N}\approx \frac{i^{-1}}{\sum\limits_{k=1}^m k^{-1}}.$$
        Note that $\sum\limits_{k=1}^m k^{-1}\approx \ln m$ and take $i$ to be $\frac{m}{2}$, we complete our proof.
    
    \end{proof}
%\xy{cite zips law reference}
    \begin{example}
        In the QC news dataset~\cite{qc}, which has $m=20317$ words in its alphabet and $N=205927$ words in total, 
        % in the training dataset,
        the median of $\{t_i\}$ is approximated to 
        $\frac{2N}{m\cdot \ln m}\approx 2$.
        This indicates that about half of $E_\sigma$ are constructed with no more than $2$ transitions.
        In practice,
        % practical, 
        the number of abstract states is usually far more than the transition numbers of these words,
        making most rows 
        % for their 
        of their transition matrices \textit{missing rows}.
    \end{example}

    Filling the missing row with $\vec 0 $ is a simple solution, since no information were provided from the transitions.
    However, as estimated above, this solution will 
    % cause that 
    lead to the problem of transition sparsity, i.e.,
    the transition matrices for uncommon words are nearly null.
    Consequently,
    if the input sequence includes some uncommon words,
    the weights over states tend to vanish.
    % may 
    %when reading a sentence, which often occurs when the sentence includes some uncommon words.\
    %\xy{This sentence is confusing, rephrase in a more clear way.}
    We refer to this solution as \textit{null filling}.
    
    Another simple idea is to use the uniform distribution over states for fairness.
    In \cite{weiss2019}, the
    % applied 
    % this kind of 
    uniform distribution is used as the transition distribution
    for unseen tokens
    in the context of formal language tasks.
    However, for natural language processing,
    this solution still loses information of the current word, despite that it avoids the weight vanishment over states.
    We refer to this solution as \textit{uniform filling}.
    \cite{zhang2021} uses the \textit{synonym} transition distribution for an unseen token at a certain state. However, it increases the computation overhead when performing inference on test data, since it requires to calculate and sort the distance between the available tokens at a certain state and the unseen token.
    
    \zm{To this end, we propose a novel approach to constructing the transition matrices based on two empirical observations.
    First, each abstract state $\hat s \in \hat S$ has explicit semantics, i.e. the probabilistic distribution over labels,
    and similar abstract states tend to share more similar transition behaviours. 
    \xy{The semantic distance between abstract states is defined as follows.}
    }
    
    %To this end, we propose an empirical approach to constructing the transition matrices with reference to the semantics, i.e. the probabilistic distribution over labels, of the abstract states in $\hat S$. Specifically, we evaluate the similarity between all abstract states once and for all, and fill the missing rows based on the transition distribution in other states.

    \begin{definition}[State Distance]
    % [Distance Between States]
        For two abstract states $\hat s_1$ and $\hat s_2$, the distance between $\hat s_1$ and $\hat s_2 $ is
        defined by the Euclidean distance between their center: $$dist(\hat s_1, \hat s_2) = \lVert \rho(\hat s_1)-\rho(\hat s_2) \rVert_2.$$
    \end{definition}

    % Next, 
    We calculate the distance between 
    % every two 
    each pair of abstract states,
    % and save them in the 
    which forms a \textit{distance matrix} $M$ 
    where each element 
    $M[i,j] = dist(\hat s_i,\hat s_j)$ for $1\le i,j\le |\hat S|$.
    %\xy{Using the same notation D for different things, make them distinct and consistent throughout the paper.}
    \zm{For a missing row in $E_\sigma$, following the heuristics that 
    similar abstract states are more likely to have similar behaviours,
    we observe the transition behaviours 
    from other abstract states,
    and simulate the missing transition behaviours
    weighted by distance between states.
   Particularly, in order to avoid numerical underflow, we leverage \textit{softmin} on distance to 
    % give more weight 
    % to similar states.
    bias the weight to states that share more similarity.}
    % Therefore, 
    Formally,
    for a missing row $E_\sigma[i]$, the weight of information 
    % provided 
    % from
    set for
    another row $E_\sigma[j]$ is defined by $e^{-M[i,j]}$.

    \zm{Second, it is also observed that sometimes the RNN just remains in the current state after reading a certain word.
    Intuitively, this is because part of words in the sentence do not deliver significant information in the task.
    Therefore, we consider simulating
    behaviours from other states whilst remaining 
    in the current state with a certain probability.}
    
    %However, the transition behaviours triggered from other states are not always trustworthy.Simply following other states may cause the overfitting problem and result in precision loss. Empirically, it is observed that sometimes the RNN just remains in the current state after reading a certain word,since this word may not deliver significant information in the task. Therefore, we consider simulating behaviours from other states whilst remaining in the current state with a certain probability.
    
    \zm{In order to balance the trade-off between referring to behaviours from other states and remaining still,}
    %For this purpose, 
    we introduce a hyper-parameter $\beta$ named \textit{reference rate}, such that when WFA 
    % meets 
    is faced with a missing row,
    it has a probability of $\beta$ to refer 
    % the behaviours 
    to the transition behaviours
    from other states,
    and 
     in the meanwhile
    has a probability of $1-\beta$ to keep still.
    % \zm{Note that the selection of $\beta$ can be correlated to the proportion of transitions $(\hat s, \sigma, \hat s')$ in $\mathcal T$ where $\hat s=\hat s'$.}
    {We select the parameter $\beta$ according to the proportion of self-transitions, i.e., transitions $(\hat s, \sigma, \hat s')$ in $\mathcal T$ where $\hat s=\hat s'$.}

    To sum up, the complete 
    % formula 
    transition rule
    for the missing row is 
    \begin{equation}
    E_\sigma[i,j] = \beta\cdot \frac{\sum\limits_{k=1}^n e^{- M[i,k]} \cdot\hat T _\sigma[k,j]}{\sum\limits_{l=1}^n\sum\limits_{k=1}^n e^{- M[i,k]} \cdot\hat T _\sigma[k,l]} + (1-\beta)\cdot\delta_{i,j}.\label{missing row}
    \end{equation}
    Here $\delta_{i,j}$ is the Kronecker symbol:
    $$\delta_{i,j} =\begin{cases}1, & j=i\\0, & j\ne i\end{cases}.$$
    In practice, we can calculate $\sum\limits_{k=1}^n e^{- M[i,k]} \cdot\hat T _\sigma[k,j]$ for each $j$ and 
    % divided by their sum,
    then make division on their summation once and for all,
    which can reduce the computation overhead on transition rule extraction. 
    % time consumption.
    % In Section 5, we will compare our approach, named \textit{empirical filling}, with the null filling approach and the uniform filling approach.

\subsection{Context-Awareness Enhancement}
\label{subsec:context}
    % \subsubsection{Motivation.}
    
    % In 
    For NLP tasks, the memorization of long-term context information is crucial.
    % One of the advantages of RNN is its ability of context-awareness.
    One of the advantages of RNN and its advanced design LSTM networks is 
    the ability to capture long-term dependency.
    We 
    % hope 
    expect the extracted WFA to simulate the step-wise behaviours of RNNs whilst 
    % remaining 
    keeping track of context information  along with the state transition.
    % with a certain rate.
    To this end, we propose an approach to adjusting the transition matrix
    % novel context-awareness enhancement approach,
    such that
    the WFA 
    % may decide to 
    can remain in the current state with a certain probability.
%\xy{approach to ``doing'', if see any other one, change to the right grammar}

    Specifically, we select a hyper-parameter $\alpha\in [0,1]$ as the \textit{static probability}.
    % i.e., the probability to remain in the 
    For each word $\sigma\in\Sigma$ and its transition matrix $E_\sigma$,
    we replace the matrix with the \textit{context-awareness enhanced matrix} $\hat E _\sigma$ 
    % where
    as follows:
    \begin{equation}
        \hat E _\sigma = \alpha\cdot I_n + (1-\alpha)\cdot E_\sigma
        \label{context}
    \end{equation}
    where $I_n$ is the identity matrix. 
    
    The context-awareness enhanced matrix has explicit semantics.
    % which means that 
    When the WFA is in state $\hat s_i$ and 
    ready to process a new word $\sigma$,
    it has a probability of $\alpha$ (the \textit{static probability}) to remain in $\hat s_i$,
    or follows the original transition distribution $E_\sigma[i,j]$ with a probability $1-\alpha$.

    Here we present an illustration of how context-awareness enhanced matrices deliver long-term context information.
    Suppose that a context-awareness enhanced WFA $\mathcal A$ is 
    % executing 
    processing a sentence $w\in \Sigma^*$ with length $|w|$.
    We denote $d_i$ as the distribution over all abstract states after $\mathcal A$ reads the prefix $w[:i]$, 
    and
    particularly $d_0=I$ is the initial vector of $\mathcal A$.
    % ;
    We use $Z_i$ to denote 
    the decision made by $\mathcal A$ 
    % with 
    based on $d_{i-1}$ and the original transition matrix $E_{w_i}$.
    Formally, $d_i = d_{i-1}\cdot \hat E_{w_i}$ and $Z_i = d_{i-1}\cdot E_{w_i}$.

    The $d_i$ can be regarded as the information 
    % provided 
    obtained from the prefix $w[:i]$ by $\mathcal A$ before it 
    % reads 
    consumes $w_{i+1}$,
    and $Z_i$ can be considered as the decision made by $\mathcal A$ after it reads $w_{i}$. %\xy{a bit confusing, what's the point that needs to emphasize?}
    \begin{proposition}
        The $i$-th step-wise information $d_i$  delivered by
        processing
        $w[:i]$ contains the decision information $Z_j$
        of prefix $w[:j]$ with a proportion of $(1-\alpha)\cdot \alpha^{i-j}$, $1\le j\le i$.
    \end{proposition}
    \begin{proof}
    Since $\hat E_{w_i} = \alpha\cdot I_n + (1-\alpha)\cdot E_{w_i}$, we can calculate that
    \begin{equation}
        d_i = d_{i-1} \cdot \hat E_{w_i} = d_{i-1}\cdot [\alpha\cdot I_n + (1-\alpha)\cdot E_{w_i}]=\alpha\cdot d_{i-1} + (1-\alpha)\cdot Z_i.\label{recursion}
    \end{equation}
    Using (\ref{recursion}) recursively, we have $$d_i = (1-\alpha)\sum_{k=1}^{i} \alpha^{i-k}\cdot Z_k + \alpha^i \cdot I.$$

    \end{proof}
%\xy{change di inthe proof to the new notation}

    This shows the information delivered by $w[:i]$ 
    % contains 
    refers to the decision made by $\mathcal A$ on each prefix included in $w[:i]$,
    and the portion vanishes exponentially.
    % In Section 5, we will see that context-awareness enhanced WFA presents better precision of consisting with RNN.
    The effectiveness of the
    context-awareness enhancement method for transition matrix adjustment will be discussed in Section \ref{sec:experiments}.
    % enhanced WFA presents better precision of consisting with RNN.

    The following example presents 
    % a complete process 
    the complete approach for transition rule extraction, i.e.,  
    to generate the transition matrix $\hat E_\sigma$ with the missing row filled in and context enhanced,
    from the count matrix $\hat T_\sigma$ for a word $\sigma\in\Sigma$.
    \begin{example}
    
    % In a simplified edition, 
    Assume that there are three abstract states in $\hat S = \{\hat s_1,\hat s_2,\hat s_3\}.$
        Suppose the count matrix for $\sigma$ is $\hat T_\sigma$.
        $$
        \hat T_\sigma = \begin{bmatrix} 1 & 3 & 0 \\ 1 & 1 & 0 \\ 0 & 0 & 0\end{bmatrix},
        E_\sigma = \begin{bmatrix}0.25 & 0.75 & 0\\ 0.5 & 0.5 & 0 \\ 0.15 & 0.35 & 0.5\end{bmatrix},
        \hat E_\sigma = \begin{bmatrix}0.4 & 0.6 & 0 \\ 0.4 & 0.6 & 0 \\ 0.12 & 0.28 & 0.6 \end{bmatrix}.
        $$
        For the first two rows (states), there exist transitions for $\sigma$,
        thus we can calculate the transition distribution of
        these two rows in $E_\sigma$ 
        % with 
        in the usual way.
        However, the third row is a \textit{missing row}. 
        We set the \textit{reference rate} as $\beta=0.5$,
        and suppose that the distance between states satisfies $e^{-M[1,3]}= 2 e^{-M[2,3]}$, generally indicating
        the distance between $\hat s_1$ and $\hat s_3$ is nearer than $\hat s_2$ and $\hat s_3$.
        With the transitions from $\hat s_1$ and $\hat s_2$, we can complement the transition rule of the third row in $E_\sigma$ through (\ref{missing row}).
        The result shows that the behavior from $\hat s_3$ is more similar to $\hat s_1$ than $\hat s_2$, due to the 
        smaller distance.
        Finally, we construct $\hat E_\sigma$ with $E_\sigma$. Here we set the \textit{static probability} $\alpha = 0.2$,
        thus $\hat E_\sigma = 0.2\cdot I_3 + 0.8\cdot E_\sigma$. The result shows that the WFA with $\hat E_\sigma$ has higher probability
        to remain in the current state after consuming $\sigma$, which can preserve more information from the prefix before $\sigma$.

    \end{example}
    
\subsection{Data Augmentation}
\label{sec:augmentation}
Our proposed approach for transition rule extraction provides a solution to the transition sparsity problem.
Still, we hope to learn more dynamic transition behaviours from the target RNN, 
especially for the words with relatively low frequency to characterize their transition dynamics sufficiently based on the finite data samples.
Different from formal languages, we can generate more 
natural language samples
automatically,
as long as 
the augmented sequential data are reasonable with clear semantics and compatible with the original learning task.
Based on the augmented samples, we are able to track more behaviours of the RNN 
and build the abstract model with higher precision.
In this section, we 
introduce two data augmentation tactics 
for natural language processing tasks: \textit{Synonym Replacement}
and \textit{Dropout}.

\paragraph{Synonym Replacement} 
Based on the distance quantization among the word embedding vectors,
we can 
obtain a list of synonyms for each word in $\Sigma$.
\xy{For a word $\sigma\in\Sigma$, the \textit{synonyms} of $w$ are 
defined as the 
top-$k$ most similar words of $\sigma$ in $\Sigma$, where $k$ is a hyper-parameter and we set $k$ to 5 by default based on an empirical observation that
% as we observe that 
top-5 similar words are sufficiently reasonable to keep the semantics. 
The similarity among the words is calculated
based on the Euclidean distance between the word embedding vectors over $\Sigma$.}

\alt{Given a dataset $\mathcal D_0$ over $\Sigma$, for each sentence $w\in \mathcal D_0$,
we generate a new sentence $w'$ by replacing some words in $w$ with their synonyms.
Specifically, each word is replaced by a randomly selected synonym in its top-$k$ synonyms with probability $p_r$ (0.4 by default).}

%Note that we hope that the uncommon words in $\Sigma$ should appear more times, so as to gather more dynamic behaviours of RNNs when processing such words. Therefore, we set the probability that a certain word $\sigma\in w$ gets replaced to be in a negative correlation to its frequency of occurrence, i.e. the $i$-th most frequent word is replaced with a probability $\frac{1}{i+1}$.

\paragraph{Dropout}
Inspired by the regularization
 technique \textit{dropout},
we also propose a similar tactic to generate new sentences from $\mathcal D_0$.
Initially, 
we introduce a new word named \textit{unknown word} and denote it as $\left\langle \textbf{unk}\right\rangle$.
\alt{For the sentence $w\in\mathcal D_0$ that has been processed by synonym replacing,
we further replace the words that haven't been replaced with $\left\langle \textbf{unk}\right\rangle$ 
with a certain probability $p_d$ (0.2 by default)}.
Finally, new sentences generated by both synonym replacement and dropout form the augmented dataset $\mathcal D$.

With the dropout tactic, we can observe the behaviours of RNNs when it 
processes an unknown word $\hat\sigma\not\in \Sigma$
that hasn't appeared in $\mathcal D_0$. 
Therefore, the extracted WFA can also have better generalization ability.
\alt{The complete pipeline of the data augmentation algorithm is elaborated in Algorithm~\ref{data_augmentation_alg}.
\xy{Note that the $rand()$ function samples from $[0, 1]$ in a uniform manner.}}

\begin{algorithm}[H]
% \SetAlgoLined \KwResult{Augmented dataset $\mathcal D$}
\alt{
\SetKwInOut{Input}{Input}
\SetKwInOut{Output}{Output} 
\Input{Original dataset $\mathcal D_0$, hyper-parameter $k=5$, $p_r=0.4$, $p_d=0.2$} 
\Output{Augmented dataset $\mathcal D$}
Obtain the synonyms $\sigma_1, \sigma_2, \cdots, \sigma_k$ of each word $\sigma\in w$ in the vocabulary of $\mathcal D$\;
$\mathcal D\gets \{\}$\;
\For{each sentence $w\in \mathcal D_0$}{
\For{each word $\sigma\in w$}{
\If{$rand()<p_r$}{
Replace $\sigma$ with selected synonym from $\{\sigma_1, \sigma_2, \cdots, \sigma_k\}$\;
}\Else{
\If{$rand()<p_d$}{
Replace $\sigma$ with $\left\langle \textbf{unk}\right\rangle$\;
}}
}
Obtain a new sentence $w'$, and add $w’$ to $\mathcal D$\;}
\textbf{return} $\mathcal D$\;
}
\caption{\alt{Data Augmentation for Transition Rule Extraction}}
\label{data_augmentation_alg}
\end{algorithm}
\vspace{0.5cm}
\xy{We illustrate the above data augmentation algorithm using the following example to generate a new sentence $w'$ from $\mathcal D_0$.}

\begin{example}
    Consider a
    sentence $w$ from the original training set $\mathcal D_0$,
    $w=$[`I', `really', `like', `this', `movie']. First, the word `like' is chosen to be replaced by one of its synonym `appreciate'.
    Next, the word `really' is dropped from the sentence, i.e. replaced by the unknown word $\langle \textbf{unk}\rangle$.
    Finally, we get a new sentence $w'=$[`I', `$\langle \textbf{unk}\rangle$', `appreciate', `this', `movie'] and 
    put it into the
    augmented dataset $\mathcal D$.

    Since the word `appreciate' may be an uncommon word in $\Sigma$, we can capture new transition information provided by RNNs. We can also capture the behavior of RNN when it reads an unknown word after the prefix [`I'].
\end{example}

\alt{Note that the role of \textit{data augmentation} in our extraction approach is different from that used in the training phase of RNNs. While data augmentation used in the training phase aims to improve the performance of RNNs, the goal of data augmentation in this work is to improve the WFA extraction precision. To this end, we use data augmentation in the testing phase to extract more transition dynamics to construct the abstract model.}
% \section{Experiments}
% \section{Extraction Precision Evaluation}
\section{Evaluation}
\label{sec:experiments}
In this section, we 
% apply 
evaluate our extraction approach on two 
natural language
datasets and demonstrate its performance 
on 
% in terms of 
precision and scalability.

\subsection{Datasets and RNNs} 
    We select two popular datasets for NLP tasks and train the target RNNs on them.
    %\xy{Rephrase this sentence, maybe split it into two sentences.}
    \begin{enumerate}
        \item The CogComp QC Dataset (abbrev. QC)~\cite{qc} contains news titles which are labeled with different topics.
        The dataset is divided into 
        a training set containing 20k samples and 
        a test set containing 8k samples.
        Each sample is labeled with one of seven categories.
        We train an LSTM model $\mathcal R$ on 
        the training set, which achieves an accuracy of $81\%$
        on the test set.

        \item The Jigsaw Toxic Comment Dataset (abbrev. Toxic)~\cite{toxic} contains comments from Wikipedia's talk page edits,
        with each comment labeled 
        toxic or not. 
        We select 25k non-toxic samples and toxic samples respectively,
        and divide them into the training set and test set in a ratio of four to one. 
        % Similarly, 
        We train an LSTM model which achieves $90\%$ accuracy on the test set.
    \end{enumerate}
%\xy{Use consistent tense for the dataset and model introduction }

    \paragraph{Metrics} 
    % For the purpose of 
    % simulate 
    % representing the behaviours of RNNs better,
    \alt{We use \textit{Consistency Rate (CR)} 
    \textit{and Jensen–Shannon Divergence (JSD)}
    as our evaluation metrics.
    For a sentence in the test set $w\in \mathcal{D}_{test}$,
    we use $\mathcal R(w)[i]$ and $\mathcal A(w)[i]$ to denote the prediction score on class $i$ of the RNNs and WFA, respectively.
    The \textit{Consistency Rate} measures the consistency of the output decision between the two models, which
    is formally defined as}
    \begin{equation}
    CR = \frac{|\{w\in D_{test}:\arg\max\limits_{i} \mathcal A (w)[i] = \arg\max\limits_{i} \mathcal R (w)[i]\}|}{|\mathcal D_{test}|}.        
    \end{equation}

    \alt{
    The \textit{Jensen–Shannon Divergence}~\cite{menendez1997jensen} measures the distance of two probability distributions, i.e., the outputs of WFA and RNN, 
    which is formally defined as    
    \begin{equation}
        JSD = \frac{1}{2}\sum_i \xy{(}\mathcal A (w)[i]\log(\frac{2\mathcal A (w)[i]}{\mathcal A (w)[i]+\mathcal R (w)[i]}) +\mathcal R (w)[i]\log(\frac{2\mathcal R (w)[i]}{\mathcal A (w)[i]+\mathcal R (w)[i]})\xy{)}.
    \end{equation}
  Note that the Consistency Rate measures the consistency of the classification decision between the WFA and the RNN, while Jensen–Shannon Divergence evaluates the similarity of the output probability distributions between the two models. 
    These two metrics 
    % correspond to 
    evaluate the consistency between 
    % different degrees of 
    the abstract model and 
    % abstraction of 
    RNN to a different degree.
    In this paper we mainly focus on the consistency of predicted labels, hence we apply Consistency Rate as our major measurement.}

    \subsection{Missing Rows Complementing}
    As discussed in Section \ref{subsec:missing_rows},
    % 3.2, 
    we take two approaches as baselines, the \textit{null filling}
    and the \textit{uniform filling}. 
    The extracted  WFA with these two approaches are denoted as $\mathcal A_0$
    and $\mathcal A_U$, respectively. 
    The WFA extracted by our \textit{empirical filling} approach is denoted as $\mathcal A_E$. 

    \begin{table}
        \centering
        \setlength{\tabcolsep}{3mm}{
        \begin{tabular}{|c|c|c|c|c|c|c|}
            \hline
            {Dataset} &
            
            \multicolumn{3}{|c|}{QC} & \multicolumn{3}{c|}{Toxic} \\
            \hline
            %\cline{2-5}
            Metric
             & \alt{CR($\uparrow$)} & \alt{JSD($\downarrow$)} & 
              Time(s) &
             \alt{CR($\uparrow$)} & \alt{JSD($\downarrow$)}
             & Time(s)  \\
             %& \alt{CR($\uparrow$)} & \alt{JSD($\downarrow$)} \\
            \hline
             $\mathcal{A}_0$ 
             & \alt{0.26} & \alt{0.25} & 47
             & \alt{0.57} & \alt{0.09} & 167 \\
            \hline
            $\mathcal{A}_U$ 
             & \alt{0.60} & \alt{0.21} & 56
             & \alt{0.86} & \alt{0.06} & 180 \\
            \hline
            $\mathcal{A}_E$
            & \alt{\textbf{0.80}} & \alt{\textbf{0.10}} & 70
            & \alt{\textbf{0.91}} & \alt{\textbf{0.02}} & 200 \\
            \hline
            \end{tabular}}
        \caption{Evaluation results of different filling approaches on missing rows.}
        \label{tab:missing_row_cr}
    \end{table}

    Table \ref{tab:missing_row_cr} shows the evaluation results of three rule filling approaches.
    We conduct the comparison experiments on QC and Toxic datasets 
    % mentioned above,
    and select the cluster number for state abstraction 
    % abstracting states 
    as $40$ and $20$ for the QC and Toxic datasets, respectively.
    
    The three rows labeled with the type of WFA show the 
    evaluation results of different approaches. 
    % consisting with their filling approach.
    For the $\mathcal A_0$ based on nul filling 
    , the WFA returns the weight of most sentences in $\mathcal D$ with $\vec 0$,
    which fails to provide sufficient information for prediction. 
    For the QC dataset, only a quarter of sentences in the test set are classified correctly.
    The second row 
    shows that the performance of $\mathcal A_U$ is better than $\mathcal A_0$.
    % but it's still unsatisfactory.
    The last row presents the evaluation result of $\mathcal A_E$, which fills in the missing rows by our approach.
    In this experiment, the hyper-parameter \textit{reference rate} is set as $\beta=0.3$.
    % \xy{Move all parameter settings to form a paragraph after the ``Datasets and RNNs models'' introduction}
    We can see that our empirical approach achieves significantly better accuracy, which is $20\%$ and $5\%$ higher than uniform filling on the two datasets, respectively.
    \alt{As for JSD, we can see that our empirical approach also outperforms the baselines notably %\xyc{significantly? change it to concrete number?}
    over both QC and Toxic datasets.}

    The columns labeled \textit{Time} show the execution time of the whole extraction workflow, from tracking transitions to evaluation on test set, but not include the training time of RNNs.
    % The time consumption of $\mathcal A_E$ is not much more than $\mathcal{A}_U$ and $\mathcal A_0$.
We can see that the extraction overhead of our approach ($\mathcal A_E$) is
about 
the same as 
$\mathcal{A}_U$ and $\mathcal A_0$.

    \subsection{Context-Awareness Enhancement}
    In this experiment, we leverage the context-awareness enhanced matrices when constructing the WFA.
    We adopt 
    % take 
    the same configuration on cluster numbers $n$ as the comparison experiments above, i.e.
    $n=40$ and $n=20$. 
\xy{The experiment results
    are summarized in Table \ref{context_cr}.
    The columns titled \textit{Config.}
    indicate
    if the extracted WFA leverage context-awareness matrices.}
    We also take the WFA with different filling approaches, the uniform filling and empirical filling, into comparison.
    % \zm{(which are far better than null filling)}
    {Experiments on null filling is omitted due to limited precision.}

    \begin{table}
        \centering
        {
        \begin{tabular}{|c|c|c|c|c|c|c|c|}
            \hline
            \multirow{2}{*}{\alt{Rule}}& \multirow{2}{*}{\alt{Config.}} & \multicolumn{3}{c|}{QC} & \multicolumn{3}{c|}{Toxic}\\
            \cline{3-8}
            & & \alt{CR($\uparrow$)} & \alt{JSD($\downarrow$)} & 
              Time(s) & \alt{CR($\uparrow$)} & \alt{JSD($\downarrow$)} & Time(s)\\
              \hline
             %\cline{2-8}
            \multirow{2}{*}{$\mathcal{A}_U$} 
            & None 
            & \alt{0.60} & \alt{\textbf{0.21}} & 56
            & \alt{0.86} & \alt{\textbf{0.06}} & 180 \\
            \cline{2-8}
            & Context 
            & \alt{\textbf{0.71}} & \alt{{0.22}} & 64
            & \alt{\textbf{0.89}} & \alt{{\textbf{0.06}}} & 191\\
            \hline
            \multirow{2}{*}{$\mathcal{A}_E$} 
            & None
            & \alt{{0.80}} & \alt{\textbf{0.10}} & 70
            & \alt{{0.91}} & \alt{\textbf{0.02}} & 200\\
            \cline{2-8}
            & Context 
            & \alt{\textbf{0.82}} & \alt{{0.13}} & 78
            & \alt{\textbf{0.92}} & \alt{{0.03}} & 211\\
            \hline
            \end{tabular}}
            \caption{Evaluation results of with and without context-awareness enhancement.}
            \label{context_cr}
    \end{table}

    For the QC dataset, we 
    % take 
    set the \textit{static probability} as $\alpha=0.4$. 
    The consistency rate of WFA $\mathcal A_U$ 
    improves 
    $11\%$
    with the context-awareness enhancement,
    and $\mathcal A_E$ improves $2\%$. As for the Toxic dataset, we take $\alpha=0.2$
    and the consistency rate of the two WFA improves $3\%$ and $1\%$ respectively.
   \alt{This shows that the WFA with context-awareness enhancement remains more context information from the prefixes of sentences, making it simulate RNNs' classification decision better.}
    \alt{However, the WFA equipped with context-awareness enhancement exhibit larger JSD, which is caused by the fact that context-awareness enhancement reduces the transition magnitude, since larger $\alpha$ leads to higher probability on remaining in the current state. This reveals a trade-off between the abstraction precision evaluated by decision label consistency and prediction score consistency.}

    Still, the context-awareness enhancement processing costs little time, since we only calculate the adjusting formula (\ref{context}) for each
    $E_\sigma$ in $E$. 
    The additional
    extra time consumption is 8s for the QC dataset and 11s for the Toxic dataset.

    \subsection{Data Augmentation}
    Finally, we evaluate the WFA extracted with transition behaviours from augmented data.
    Note that the two experiments above are based on the primitive dataset $\mathcal D_0$. %\xyc{Maybe remove "training" to avoid confusion on whether it is data augumentation during training phase}
    In this experiment, we leverage the data augmentation tactics to generate the augmented training set $\mathcal D$,
    and extract WFA with data samples from $\mathcal D$.
    In order to get best performance, we build WFA with context-awareness enhanced matrices.

    \begin{table}
        \centering

        {
        \begin{tabular}{|c|c|c|c|c|c|c|c|}
            \hline
            \multirow{2}{*}{\alt{Rule}}& \multirow{2}{*}{\alt{Data}} & \multicolumn{3}{c|}{QC} & \multicolumn{3}{c|}{Toxic}\\
            \cline{3-8}
            & & \alt{CR($\uparrow$)} & \alt{JSD($\downarrow$)} & 
              Time(s) & \alt{CR($\uparrow$)} & \alt{JSD($\downarrow$)} & Time(s)\\
              \hline
             %\cline{2-8}
            \multirow{2}{*}{$\mathcal{A}_U$} 
            & $\mathcal D_0$  
            & \alt{{0.71}} & \alt{{0.22}} & 64
            & \alt{{0.89}} & \alt{{{0.06}}} & 191\\
            \cline{2-8}
            & $\mathcal D$  
            & \alt{\textbf{0.76}} & \alt{{\textbf{0.18}}} & 81
            & \alt{\textbf{0.91}} & \alt{{\textbf{0.05}}} & 295\\
            \hline
            \multirow{2}{*}{$\mathcal{A}_E$} 
            & $\mathcal D_0$ 
            & \alt{{0.82}} & \alt{{0.13}} & 78
            & \alt{{0.92}} & \alt{{0.03}} & 211\\
            \cline{2-8}
            & $\mathcal D$ 
            & \alt{\textbf{0.84}} & \alt{{\textbf{0.12}}} & 85
            & \alt{\textbf{0.94}} & \alt{{\textbf{0.02}}} & 315\\
            \hline
            \end{tabular}
            }
    \caption{Evaluation results of with and without data augmentation.}
    \label{aug cr}
    \end{table}

    Table \ref{aug cr} shows the results of consistency rate of WFA extracted with and without augmented data.
    The rows labeled $\mathcal D_0$ show the results of WFA that are extracted with the primitive training set, 
    and the result from the augmented data is shown in rows labeled $\mathcal D$.
    With more transition behaviours tracked, the WFA extracted with $\mathcal D$ 
    % performs 
    demonstrates better precision.
    \alt{Specifically, the WFA extracted with
    both empirical filling and context-awareness enhancement achieves a further $2\%$ increase in consistency rate on the two datasets.
    In addition, the extractions with augmented data also exhibit better JSD.}

    To summarize, by using our transition rule extraction approach, 
    % with three optimazation approaches,
    the consistency rate of extracted WFA on the QC dataset and the Toxic dataset achieves $84\%$ and $94\%$, respectively. 
    Taking the primitive extraction algorithm with uniform filling as baseline, 
    of which experimental results 
    in terms of CR 
    are $60\%$ and $86\%$,
    our 
    % three optimazation 
    approach achieves an improvement of $22\%$ and $8\%$ in consistency rate.
    \alt{Regarding the Jensen–Shannon Divergence, though there is a little %\xyc{brittle?}
    drop made by the context-awareness enhancement, our approach still outperforms the baseline methods 
    significantly. %\xyc{for all cases?}.
    Taking uniform filling for comparison, our 
    overall approach improves the JSD from 0.21 to 0.12 on QC dataset and 0.06 to 0.02 on Toxic dataset.}
    For the time complexity, the time consumption of our approach 
    increases 
    from $56s$ to $81s$ on QC dataset,
    and from $180s$ to $315s$ on Toxic dataset.
    There is no significant time cost increase when 
    adopting our approach for complicated natural language tasks.
    We can conclude that our 
    transition rule extraction approach 
    makes better approximation of RNNs,
    and 
    is also efficient enough 
    to be applied to practical applications for large-scale 
    natural language tasks. 

%\subsection{\alt{More ablation studies}}
% \subsection{\alt{More experimental analysis}}
\xy{\subsection{Parameter Effect Evaluation}}

\xy{In this section, we conduct 
% comprehensive 
experiments to evaluate the impact of the  hyper-parameters on the validity of 
% analysis on 
our 
% proposed 
extraction approach, including the 
% selection of two hyper-parameters, 
reference rate $\beta$, static probability $\alpha$, and the number of cluster $K$.}

\paragraph{\alt{Reference rate $\beta$}}
% \xy{We perform }
\xy{We first evaluate the impact of the \textit{reference rate}. 
To this end, we set $\beta$ to different values from $\{0.1,0.3,0.5,0.7,0.9\}$. Meanwhile, we set $\alpha$ to a fixed value 0. 
The results are shown in Figure~\ref{fig:beta}, where we take the {uniform filling} as baseline (the dotted lines). 
% In most cases particularly $\beta\le 0.5$, 
We observe that our filling method outperforms uniform filling for a large range of parameter values (less than 0.7), under both CR and JSD metrics. 
%This verifies the generalization of our proposed which does not heavily rely on the hyper-parameter $\beta$.}
A relatively small $\beta$ (e.g., less than 0.5) leads to better extraction precision.
}
%\xy{Todo: revise the result analysis: according to the evaluation results, setting beta to a value e.g., less than 0.7 leads to better performance, which also shows that our approach is insensitive to this parameter choice.}

\begin{figure}[t]
    \centering
    \begin{tabular}{cc}
        \includegraphics[width=0.45\textwidth]{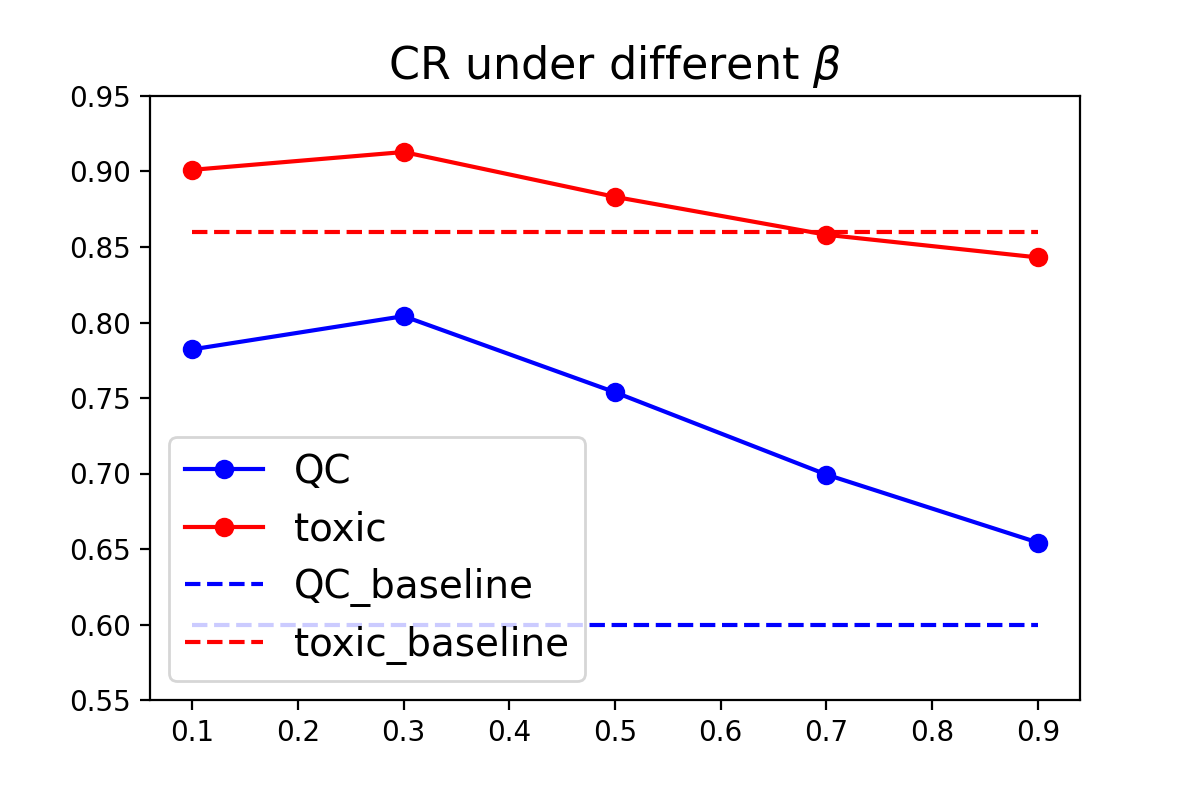}
         &  \includegraphics[width=0.45\textwidth]{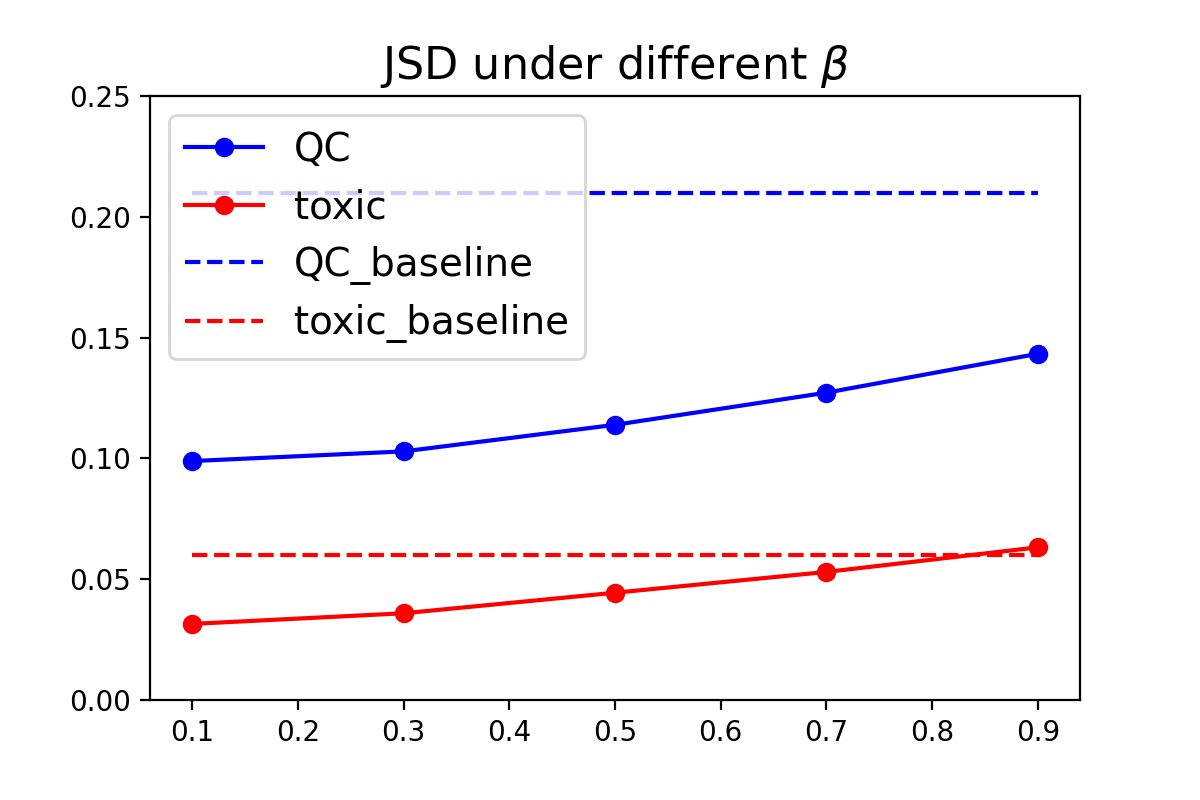} \\

    \end{tabular}
        \caption{\alt{CR and JSD on the two datasets under different $\beta$}.}
    \label{fig:beta}
\end{figure}

\paragraph{\alt{Static probability $\alpha$}}
\alt{Similarly, we conduct an experiment to evaluate the impact of different \textit{static probability} values $\alpha\in\{0,0.2,0.4,0.6,0.8,1.0\}$ \zmw{on the performance of our approach}. 
\xy{We set $\beta$ to 0.3 based on the above evaluation results}. 
%\xy{Todo: revise accordingly similar to the last paragraph.}
The results are illustrated in Figure~\ref{fig:alpha}.
%Our method still performs the {uniform filling} as baseline (the dotted lines) \zmw{for $\alpha\le 0.4$}. 
%\xy{Not always, revise based on the figure results.}
\xy{Compared with the case of $\alpha=0$ where we do not apply the context-awareness enhancement, it leads to improvements on the CRs when setting $\alpha$ to values from $\{0.2,0.4\}$. Meanwhile, as discussed before, context-awareness enhancement reduces the transition scale of WFA, which leads to performance degradation in terms of JSD.
This reveals a trade-off between CR and JSD among different selections on $\alpha$.}
% Therefore, the selection of $\alpha$ is relatively more intensive than $\beta$, and 
Based on the results, we suggest setting $\alpha$ to a small positive value (less than 0.4).
}
% provide intuition on how to select a proper alpha; point out the insensitive to a certain value range of alpha

\begin{figure}[t]
    \centering
    \begin{tabular}{cc}
        \includegraphics[width=0.45\textwidth]{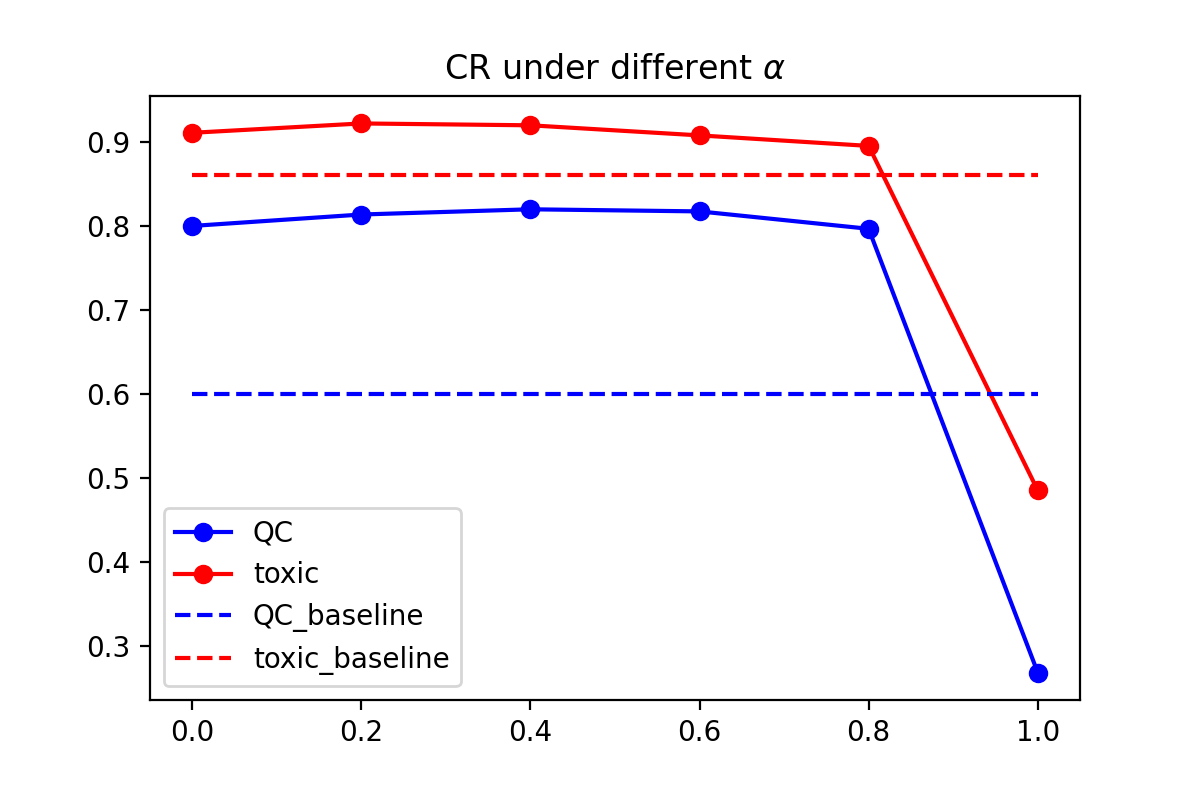}
         &  \includegraphics[width=0.45\textwidth]{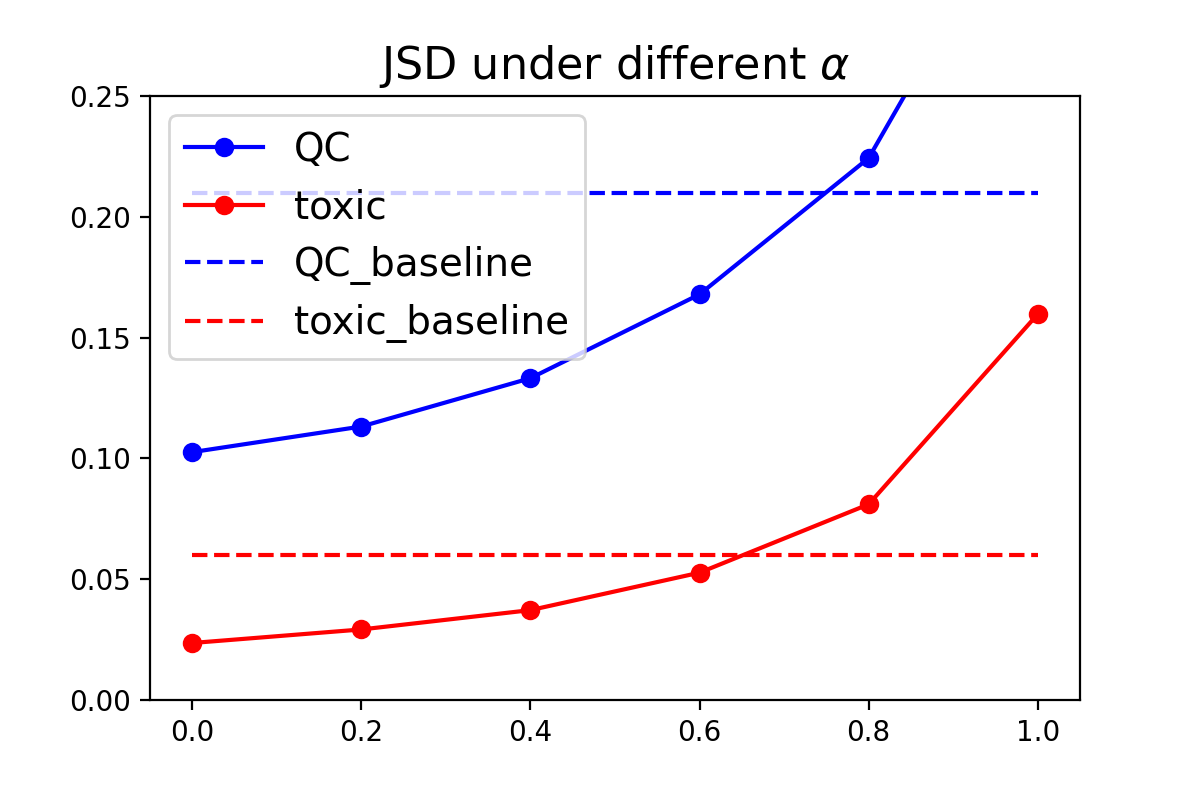} \\

    \end{tabular}
        \caption{\alt{CR and JSD on the two datasets under different $\alpha$}.}
    \label{fig:alpha}
\end{figure}

\paragraph{\alt{Cluster number $K$}}
\xy{Finally, we 
evaluate the impact of 
cluster number $K\in\{10,20,30,40,50\}$ on the performance of our approach.
The results are shown in Figure~\ref{fig:K}, where our approach is denoted by $A_E$, and the uniform filling and null filling are denoted as $A_U, A_0$ respectively. We can see that our approach outperforms the baselines in all cases. Note that as $K$ increases from 10 to 50, the performance of $A_U$ and $A_0$ consistently decreases, which is caused by the transition sparsity problem we have identified. 
% However, the performance of our method still remains better even $K$ is large, which 
\zmw{The evaluation results demonstrate the robustness  of our method against the cluster number $K$}.}
%\xy{Todo: Revise the anlysis: keep the decreasing peformance of Au and A0 due to sparsity. Point out the robustness of our approach with K.}

% \xy{'Provide intuition and guidacne on how to set K in experiments':}

% \zmw{I failed to provide this}

\begin{figure}[t]
    \centering
    \begin{tabular}{cc}
        \includegraphics[width=0.45\textwidth]{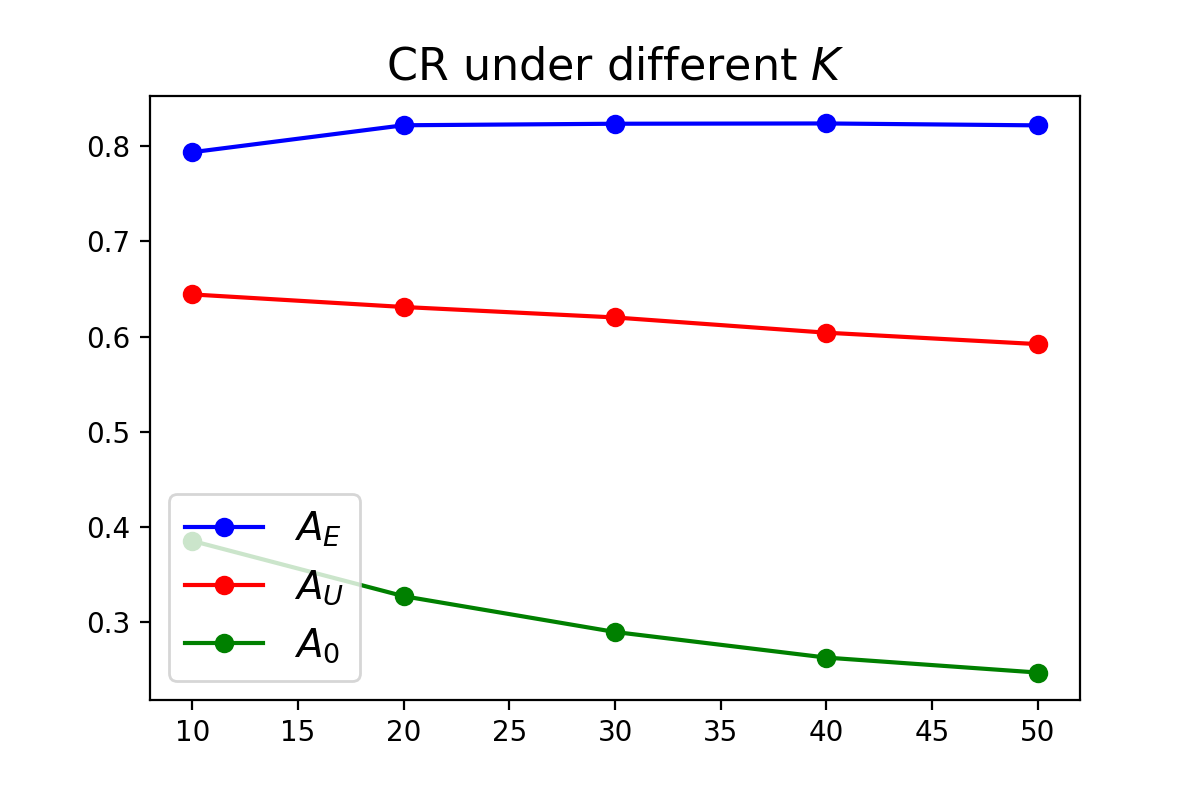}
         &  \includegraphics[width=0.45\textwidth]{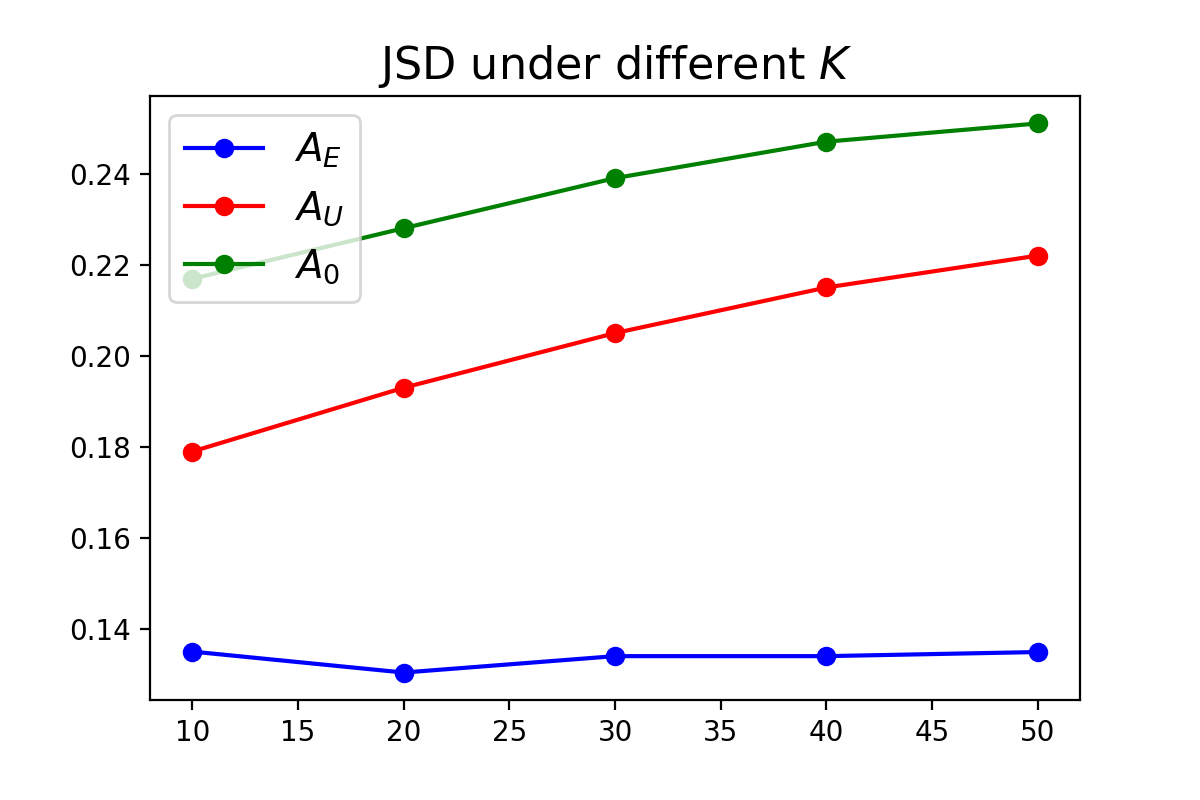} \\
        \alt{(a) CR on QC dataset} & \alt{(b) JSD on QC dataset} \\
        \includegraphics[width=0.45\textwidth]{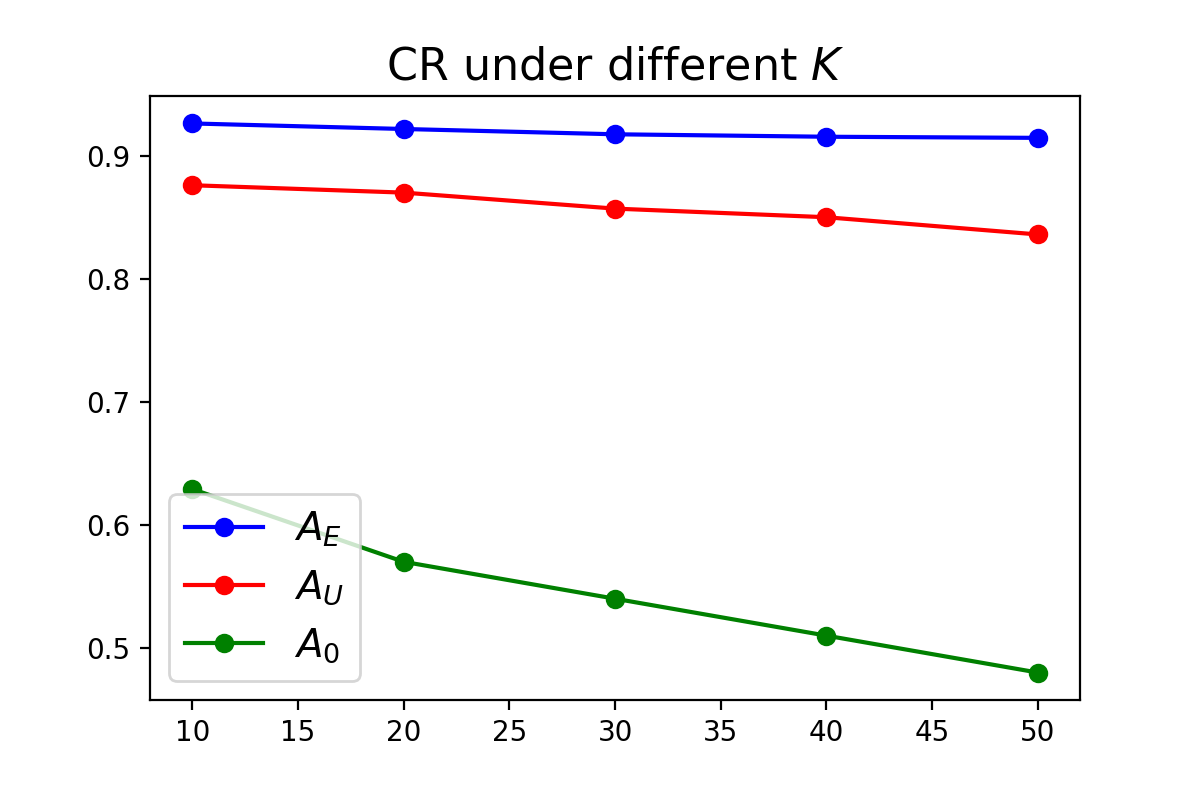}
         &  \includegraphics[width=0.45\textwidth]{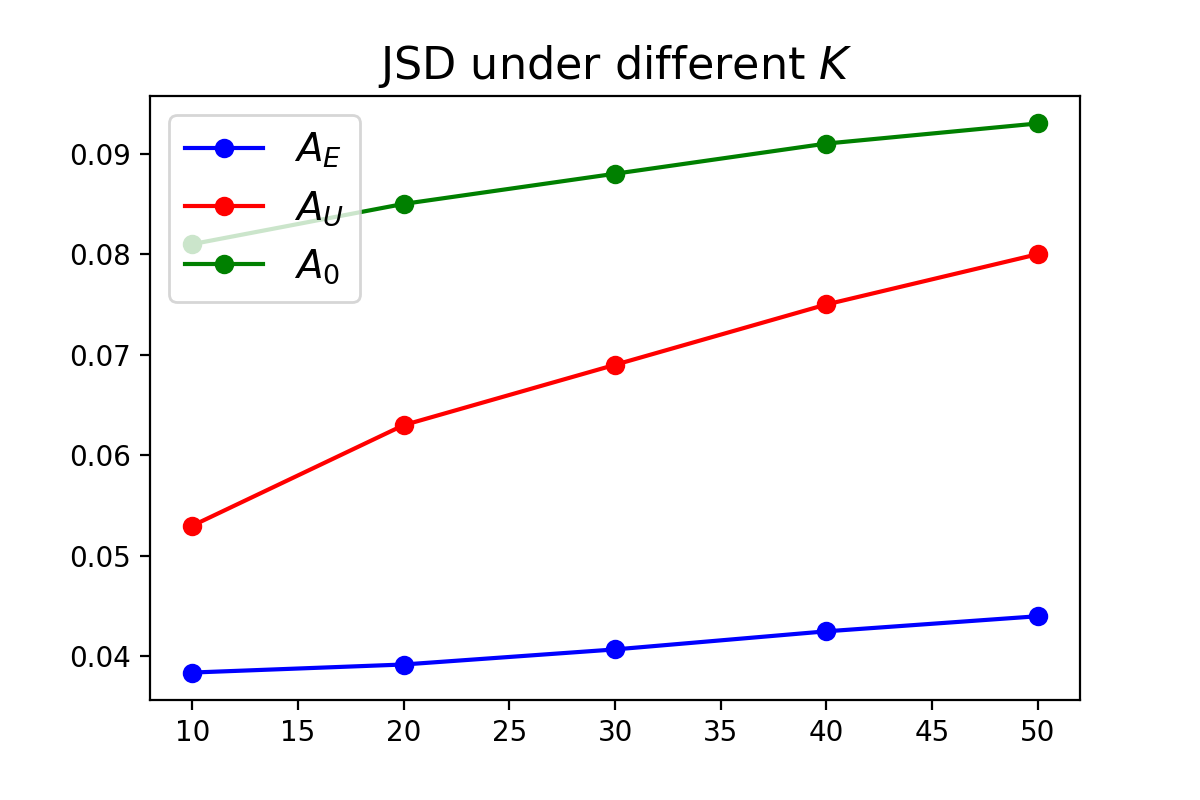} \\
        \alt{(c) CR on Toxic dataset} & \alt{(d) JSD on Toxic dataset} \\
    \end{tabular}
    \caption{\alt{Overall comparison under different $K$}.}
    \label{fig:K}
\end{figure}

\section{Weighted Automata-based Explanation of RNNs}
\label{sec:expla}
In this section,
we propose a novel explanation framework of RNNs for natural language tasks based on the extracted WFA.
In the explanation framework, we consider using  transition matrix\footnote{In this section, we focus on the final extracted transition matrix $\hat E_\sigma$, and abuse the notation $E_\sigma$ for the sake of simplicity.} $E_\sigma$
for  each word $\sigma$ as its word embedding, 
named as
% We name this 
% embedding as 
Transition Matrix Embeddings (TME).
In the following, we 
% initially 
first introduce
% discuss the explanation 
TME and explain its difference from traditional pretrained word embeddings.
Next, we present a global explanation and contrastive explanation method based on TME to interpret the behaviours of RNNs:
% demonstrate 
% two applications of TME on understanding the behaviours of RNNs:
(1) \textit{global explanation:} we use TME to calculate the word-wise attribution of RNN's decisions, 
(2) \textit{contrastive explanation:} we compare TME with the conventional word embedding method to analyze the task-oriented word semantics learned by RNNs.
% relations between words.
We also reveal two intriguing properties of RNNs identified by the contrastive method,
% with these analyses, 
and validate the effectiveness of the proposed embedding and explanation framework for RNNs by applying it to pretraining and adversarial example generation.

\subsection{Transition Matrix as Word Embeddings}
\begin{figure}[t]
    \centering
    \includegraphics[width=0.9\textwidth]{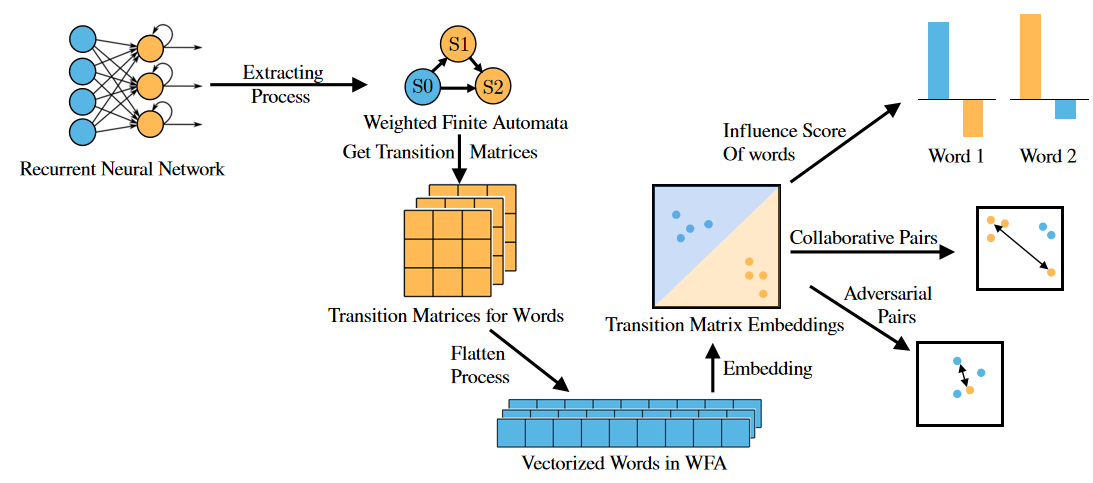}
    \caption{WFA-based explanation framework.}
    \label{fig:expla}
\end{figure}
Suppose the extracted WFA $\mathcal A$ from RNN $\mathcal{R}$ has $n=|\hat S|$ words.
For each word $\sigma$ and its corresponding transition matrix $E_\sigma\in \mathbb R^{n\times n}$ in $\mathcal A$,
recall that the $(i,j)$-th element of $E_\sigma$ represents the transition probability of $\mathcal A$ from $s_i$ to $s_j$ after reading word $\sigma$, 
which is an approximation of the transition probability of $\mathcal{R}$ between these two states.
Therefore, if two words share similar transition matrices, 
they trigger similar behaviours of RNN $\mathcal R$,
% in the input sequence,
and hence represent similar semantics from the RNN $\mathcal R$'s perspective for the current task.

This observation motivates us using the transition matrices to craft word embeddings.
In order to obtain the task-oriented embedding vector,
% match the shape of word embedding for a word which is a vector, 
we flatten the transition matrix $E_\sigma\in\mathbb R^{n\times n}$ into a  vector $\boldsymbol{e}_\sigma$ 
of $N=n^2$ dimension:

\begin{equation}
    \boldsymbol e_\sigma [j+n\cdot (i-1)] = E_\sigma [i,j],\quad \text{for} \ 1\le i,j\le n,
\end{equation}
% and name the crafted  vectors 
which we refer to
as Transition Matrix Embeddings (TME).
% It is also noteworthy 
%Note that this embedding method 
% like $\boldsymbol e_\sigma$ 
%possesses a normalization property, as each row of $E_\sigma$ sums up to 1.
%we obtain $\| \boldsymbol{e}_\sigma\|_1 = n$ for every $\sigma$.
%\xy{Revise the final sentence. What is meant to indicate?}

Note that 
TME are fundamentally different from the traditional pretrained word embeddings,\textit{ e.g.} Word2vec~\cite{mikolov2013efficient}, Gloves~\cite{pennington2014glove}.
The TME characterize
% are based on the 
transition behaviours of RNNs when processing each word, which are task oriented (\textit{e.g}., text classification), while pretrained embeddings are word-semantic oriented.
\alt{Therefore, for two words $\sigma_1$ and $\sigma_2$, they may share similar TME $\boldsymbol e_{\sigma_1}$, $\boldsymbol e_{\sigma_2}$, yet represent different semantics and hence their embedding differences in pretrained embeddings may be large}.
% far away. 
We detail 
such cases 
in Section~\ref{relation}.
Further, the applications of TME are different from the general pretrained embeddings.
The extracted TME can be seen as a global explanation of the source RNN $\mathcal R$ (see Section~\ref{6.2} for details), 
which aids us for understanding the decision logic of the source RNN.

%We can also use the extracted TME as the embeddings initialization for similar and down-stream tasks.
\subsection{Word-Wise Attribution with TME}
%\subsection{Understanding RNNs with Influence Score}
\label{6.2}
\begin{table}[t]
    \centering
    \resizebox{\textwidth}{!}{
    \begin{tabular}[width=0.9\textwidth]{|c|c|l|}
    \hline
       Label & Category & Top-10 Influencial Words\\ \hline
\multirow{2}{*}{0}&\multirow{2}{*}{Sport} &players, lockout, rangers, sox, knicks,\\ & & basketball, coach, bruins, champions, djokovic \\ \hline
\multirow{2}{*}{1}&\multirow{2}{*}{World} &libyan, pakistan, yemen, sudan, gaddafi,\\ & & syrian, egypt, mubarak, syria, afghans\\ \hline
\multirow{2}{*}{2}&\multirow{2}{*}{US} &florida, county, wildfire, layoffs, massachusetts,\\& & mid-atlantic, wildfires, cpsc, firefighters, blagojevich\\ \hline
\multirow{2}{*}{3}&\multirow{2}{*}{Business} &stocks, dollar, consumer, goldman, mortgage,\\ & & wall, companies, rosneft, s\&p, sec\\ \hline
\multirow{2}{*}{4}&\multirow{2}{*}{Health} &disease, crackers, cancer, asthma, patients,\\& & exercise, prevention, symptoms, therapy, obesity\\ \hline
\multirow{2}{*}{5}&\multirow{2}{*}{Entertainment}&idol, cannes, arthur, gotti, beaver,\\& & mccreery, mariah, lohan, baldwin, diana\\ \hline
\multirow{2}{*}{6}&\multirow{2}{*}{Sci\_tech}&3ds, playstation, icloud, software, gmail,\\& & windows, tablets, linkedin, tablet, climate\\ \hline
    \end{tabular}
    }
    \caption{Top-10 Influential Words for 7 classes in the QC news Dataset~\cite{qc}}
    \label{tab:is}
\end{table}
We introduce a global explanation method based on
% the first application of
TME for analyzing the decision attribution of the source RNN $\mathcal{R}$. 
% Our approach provides a global explanation of RNN $\mathcal{R}$ by 
We investigate the impact of each word $\sigma$ on the decision of $\mathcal{R}$ based on its TME $\boldsymbol{e}_\sigma$.
Recall that each abstract state $\hat s$ has an explicit semantic 
% that is 
represented by its center $\rho(\hat s)$,
the probability distribution of labels.
Therefore, each transition between two states, such as $\hat s_1\to \hat s_2$,
can be interpreted as a shift in the probability distribution of labels, $\rho(\hat s_1) \to \rho (\hat s_2)$.
By multiplying the transition probability between states, 
which is saved in  $\boldsymbol e_\sigma$, 
we can calculate the average variation of prediction scores among labels after $\mathcal R$ reading word $\sigma$.
More formally,  the variation contributed by the abstract transition $(\hat s_i, \sigma,\hat s_j)$ is given by $\boldsymbol e_\sigma[j+n\cdot (i-1)]\times(\rho(\hat s_j) - \rho(\hat s_i))$.

Additionally, we take into account the 
% unequal 
uneven
significance among all abstract states, where states that appear more frequently should 
be assigned larger weight.
% receive greater consideration. 
To reflect this, we calculate the frequency of each abstract state $\hat s$ as $u(\hat s)$ and 
% include it 
incorporate it into the computation of influence score as a weight.
% in the calculation of influence. 
% To summarize, 
Formally, the \textit{Influence Score} (IS) of word $\sigma$ is formulated as
\begin{equation}
    \text{IS}(\sigma) = \sum_{i=1}^n u(\hat s_i)\{ \sum_{j=1}^n \boldsymbol e_\sigma[j+n\cdot (i-1)]\times(\rho(\hat s_j) - \rho(\hat s_i))\}.
\end{equation} 

For class $i$, the $i$-th element of influence score for word $\sigma$ represents how this word impacts the decision of RNN on this class.
%This score can be intuitively interpreted as the influence on the prediction result on each class: the higher the IS of an word on a certain class is, the more convinced the RNN is to believe the sequence belongs to that class when it reads this word.
%As this interpretation has shown, 
% Therefore, we can study which words are learnt to be influential by the source RNN by investigating
%  the words have high IS on each class.
Therefore, we can investigate the influential words for the source RNN's decisions by sorting the input words in descending order of IS.
% investigating
 % the words have high IS on each class.
%to find out.  This can be an interpretation as how RNN decides the class of a sequence during reading process. 
To demonstrate the effectiveness of the proposed influence analysis method, 
% To this end, 
we compute the IS of all words in the vocabulary of the QC news dataset, and show the top-10 influential words for each class  in Table~\ref{tab:is}.
\alt{Due to the presence of inappropriate language in the Toxic dataset 
% which is not appropriate for academic articles, 
we exclude the experiment results on this dataset.}
From Table~\ref{tab:is} we can see that 
%it is obvious to see that 
for each category, its most influential words are highly correlated to that domain.
This confirms that the proposed influence score based on TME can indeed identify the input features (words) that RNN $\mathcal R$ relies on to make decisions on each class.
% This confirms what RNN $\mathcal R$ learns which words are important in classifying the sentence into a certain class.

\begin{figure}[t]
    \centering
    \begin{tabular}{cc}
    \includegraphics[width=0.4\textwidth]{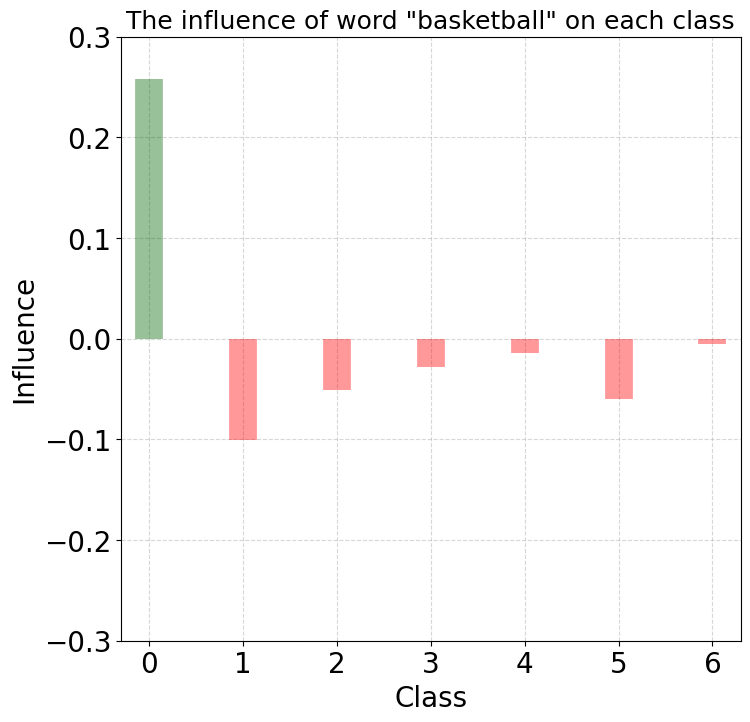} &
    \includegraphics[width=0.4\textwidth]{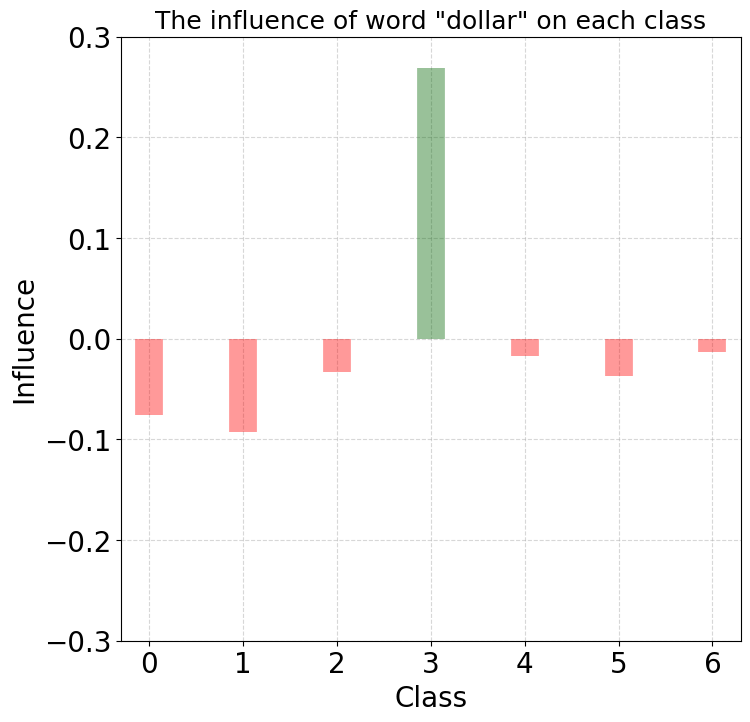} \\
    (a) Basketball & (b) Dollar\\
    \includegraphics[width=0.4\textwidth]{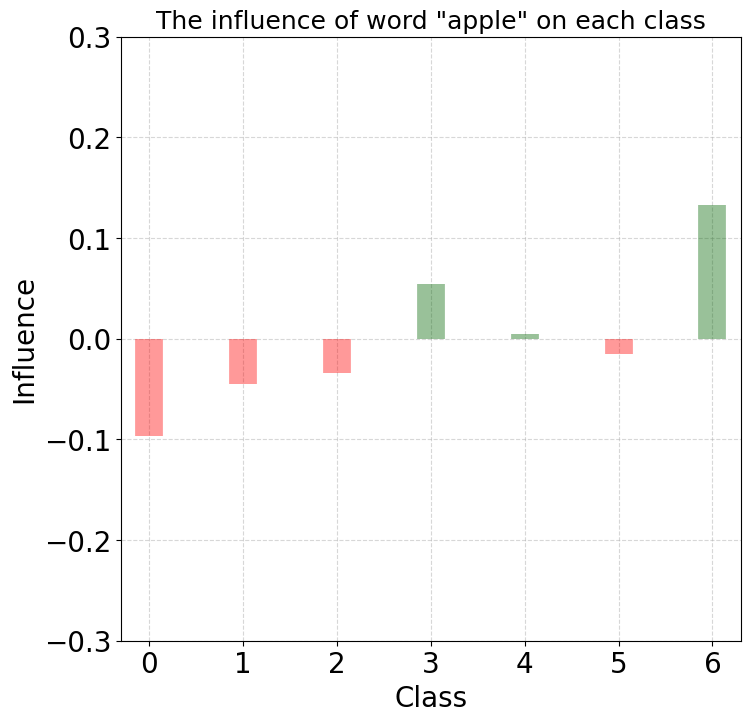} &
    \includegraphics[width=0.4\textwidth]{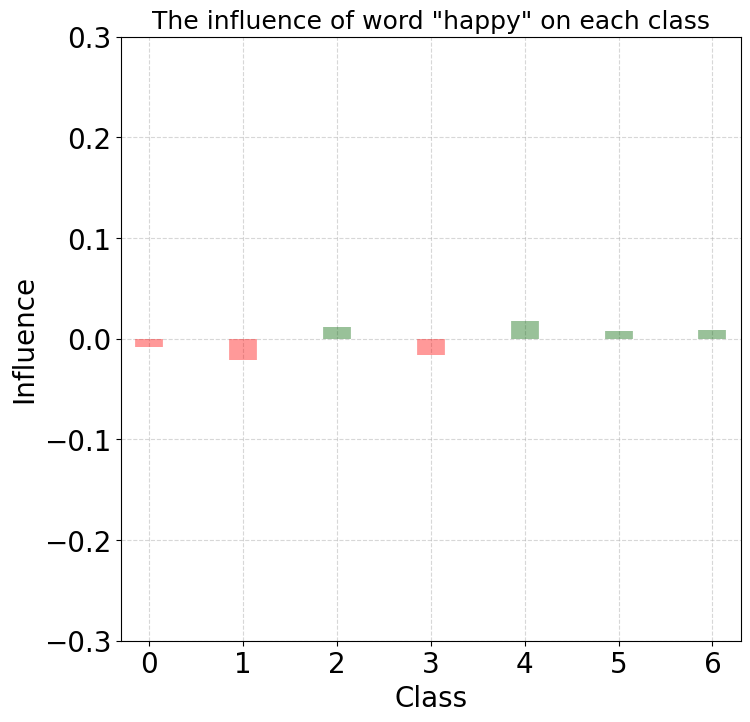} \\
    (c) Apple & (d) Happy

    \end{tabular}
    \caption{Influence Scores (ISs) visualization for some words.}
    \label{fig:is}
\end{figure}

We can also use 
the TME 
to characterize the relative importance of an 
% decision logic for 
individual word for different labels.
For instance, we show the influence scores of the words  ``Basketball'', ``Dollar'', ``Apple'', and ``Happy''  in Figure~\ref{fig:is}.
%and more visualization of influence scores can be found in Appendix ().
%To make our assertion more convinced, we show influence score of some words in Figure \ref{fig:is}. 
As shown in  Figure~\ref{fig:is}, the ISs of ``Basketball'' and ``Dollar'' demonstrate that 
they 
lead to
high prediction tendency on class ``Sport'' and ``Business'', 
respectively, which is strongly correlated to their semantic domains.
%belongs to the vocabulary of sports and business, respectively; correspondingly, as shown in figure, the IS score on the corresponding category is notably high; this implies that, when RNN reads these words, its decision is changed to believe the sequence belongs to this category. 
In contrast, the word ``Apple'' shows high influence score on class ``Business'' and ``Sci\_tech'', which is consistent with the intuition that 
this word 
% as a company 
is active in 
both business and technology news.
%As far as words that appears both in multiple fields are concerned, such as
%`apple' (both meaningful in business and technology), we observe that IS on both category is positive; as for those words that have no significant tendency on any category, for example, 'happy', the influence is naturally close to zero. As all above phenomenons fits well with the discussions we mentioned above, this can be seen as another persuasive evidence for our interpretation above.
% And finally, the word `Happy' which is of 
%  influence on classification performs uniform low IS on each classes.
Finally, the word ``Happy'', which has no particular
 influence on any class, demonstrates uniform IS on each class.

To sum up, the proposed TME 
% which is 
% based on the extracted WFA $\mathcal A$ 
provide
% can be regarded as 
a global explanation of the source RNN $\mathcal R$.
As discussed above, by 
% utilizing 
computing TME-based
influential score, we can 
% both find some 
provide explanations of $\mathcal R$ from both class-wise and word-wise perspectives.
%, which is independent of input sequences.

% \subsection{Relationship Comparison of Words}

\subsection{Contrastive Word Relation}
\label{relation}
Based on the Transition Matrix Embeddings (TME), 
we propose a contrastive method to 
% also 
investigate the relations of words in TME and conventional word embeddings, which
% and their embeddings and 
reveals two intriguing properties.

\begin{figure}[t]
    \centering
    \begin{tabular}{cc}
    \includegraphics[width=0.5\textwidth]{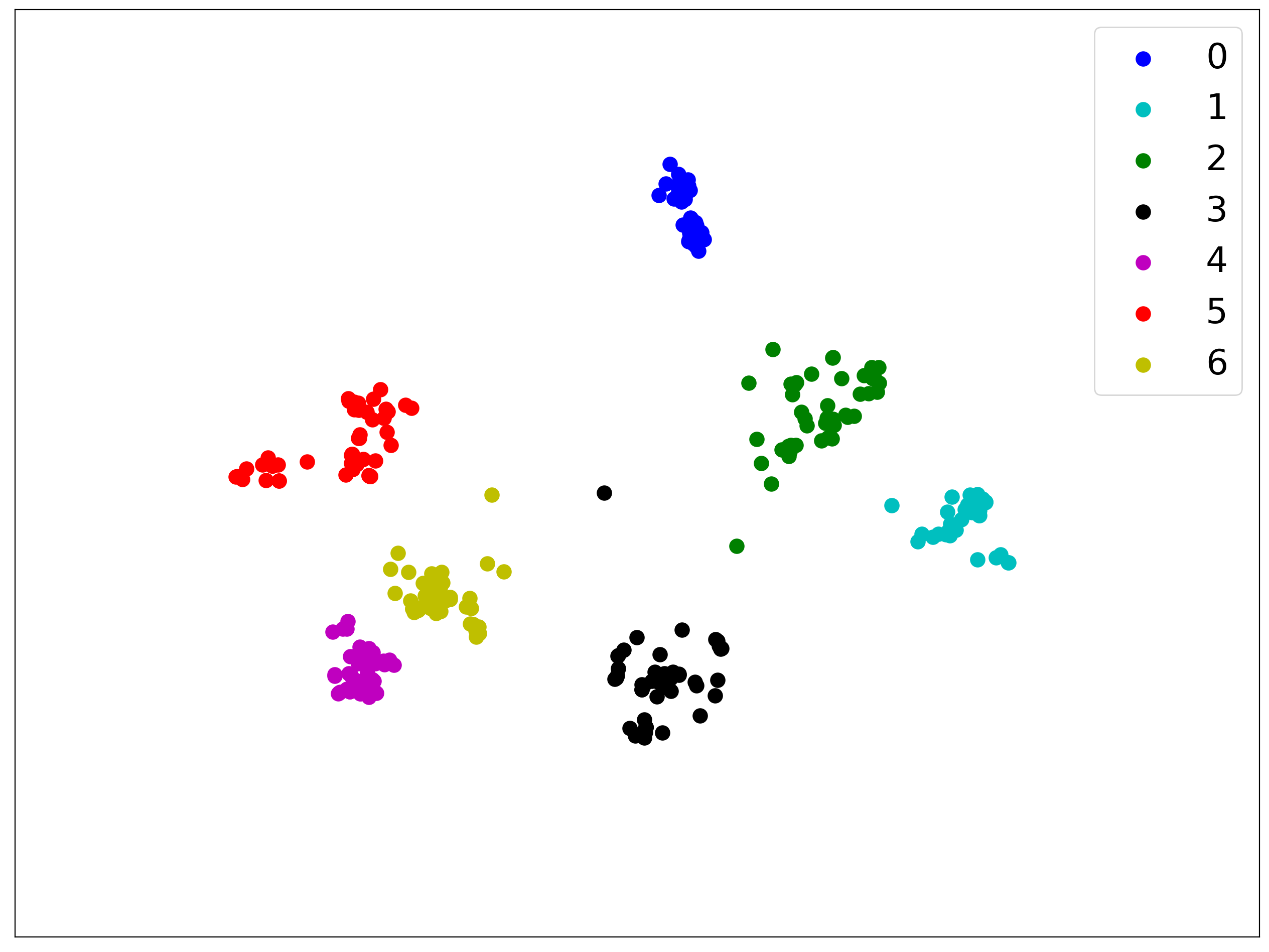}
         &      \includegraphics[width=0.5\textwidth]{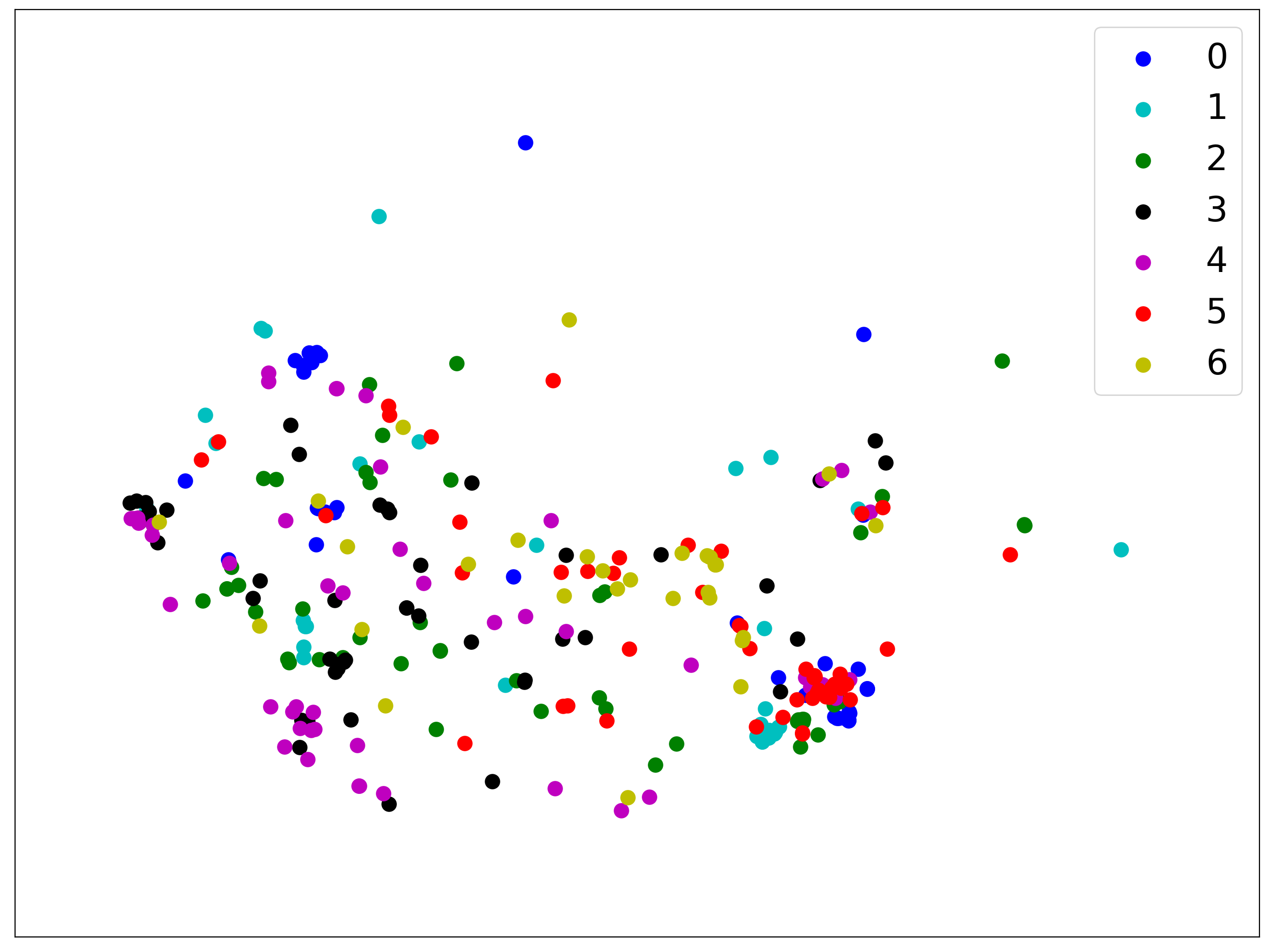}\\
        (a) TME & (b) Glove
    \end{tabular}
        \caption{t-SNE visualization of two kinds of word embeddings \alt{on QC dataset. Each color represents a class.}}
    \label{fig:tsne}
\end{figure}

First, to demonstrate the difference between TME and Glove embeddings, 
we use t-SNE~\cite{van2008visualizing} to visualize the embedding vectors of TME and Glove, which is shown in Figure~\ref{fig:tsne}.
Specifically, we select top-50 influential words from each class in QC news dataset with each color representing a class.
We can see that the selected 350 ($50 \times 7$) words demonstrate 
% distinct 
different
clustering property in these two word embeddings.
This shows our TME are quite different from 
pretrained word embeddings, wherein the word semantics are task-oriented.

%To start with, we define two distances.
We now 
% explore TME by 
characterize the contrastive relations between words in TME and the conventional word embeddings. \alt{We define $\| \cdot \|$ as the $p-$norm of a matrix or a vector divided by the number of its elements, and we set $p$ to 2. We compute the distance of two words given by TME and their conventional semantics, respectively.
For words $\sigma_1,\sigma_2$, we define their transition distance as $d_T(\sigma_1, \sigma_2) = \|\boldsymbol e _{\sigma_1} -\boldsymbol e _{\sigma_2}\|$.
In order to analyze their conventional semantic distance, we use the Glove~\cite{pennington2014glove} word embeddings $\boldsymbol g_{\sigma_1}$, $\boldsymbol g_{\sigma_2}$, and define the semantic distance as $\alt{d_S(\sigma_1, \sigma_2)=\|\boldsymbol g_{\sigma_1}-\boldsymbol g_{\sigma_2} \|}$.}
% Intriguingly, 
By analyzing these two embedding distances between words,
we find that there exist some contrastive word pairs demonstrating different properties.
We formally define two types of contrastive word pairs in the following.

% provide two definitions of some specific word pairs in the following.

\begin{definition}[$(\epsilon, \delta)$-Collaborative Pair] A $(\epsilon, \delta)$-Collaborative Pair is 
% two
a pair of words $(\sigma_1, \sigma_2)$ satisfying that
\begin{equation}
    d_T(\sigma_1, \sigma_2) \le \epsilon, \quad d_S(\sigma_1, \sigma_2) \ge \delta.
\end{equation}
\end{definition}

Here $\epsilon$ is a small positive number to guarantee
% make sure 
that the word pair $\sigma_1$ and $\sigma_2$
% share 
trigger similar transition behaviours of the source RNN $\mathcal R$.
% on them.
% However, 
On the other hand,
$\delta$ is a relatively larger positive number, indicating these two words have quite different semantics in terms of conventional embeddings.
Hence the \textit{collaborative pairs} are the word pairs that have distinct meanings, but are similar from the RNN $\mathcal R$'s perspective on the specific task.
In contrast, the \textit{adversarial pairs} are defined in a symmetry manner,
% in a different manner, 
which means the words share similar meanings, 
but are quite different from the RNN $\mathcal R$'s understanding in a particular task.

\begin{definition}[$(\epsilon, \delta)$-Adversarial Pair] 
% An $(\epsilon, \delta)$-Collaborative Pair is two words $(\sigma_1, \sigma_2)$ of 
An $(\epsilon, \delta)$-Adversarial Pair is a pair of words $(\sigma_1, \sigma_2)$ satisfying that
\begin{equation}
    \alt{d_T(\sigma_1, \sigma_2) \ge \delta, \quad d_S(\sigma_1, \sigma_2) \le \epsilon.}
\end{equation}
\end{definition}

% How do these concepts and results contribute to the explanation of the source RNN?
% In fact, those has given a further result on 
The above contrastive pairs allow us to understand
how RNN learns the semantics of the words. 
When a dataset and a task is given, the semantic of a word in the vocabulary is not learned fully obeying general word embedding, but \textit{task-oriented}. 
% That means the semantics RNNs learnt is under task and more targeted. 
% In this way, 
% RNNs do not need to learn semantics from natural language,
% but embedding in the linear space oriented by the task.
\begin{table}[t]
    \centering
   {
    \begin{tabular}[width=0.3\textwidth]{|c|c|c|c|}
    \hline
       Label & Category & Collaborative Pairs &  
       Adversarial Pairs \\ \hline
0&Sport  &(‘lakers’,‘wozniacki’)& (‘cup’,‘cups’) \\ \hline
1&World & (‘yemen’,‘gaddafi’)&(‘yemen’,‘usa’) \\ \hline
2&US  & (‘wildfire’,‘texas’)&(‘wildfire’, ‘tsunami’)\\ \hline
3&Business  & (‘wall’,‘mortage’) &(‘wall’,‘behind’)\\ \hline
4&Health &(‘therapy’,‘rice’) &(‘exercise’,‘sports’)\\ \hline
5&Entertainment &(‘cannes’,‘bieber’)&(‘diana’,‘williams’)\\ \hline
6&Sci\_tech &(‘climate’,‘software’)&(‘windows’,‘open’)\\ \hline

    \end{tabular}
    }
    \caption{Examples of 
    Collaborative Pairs and Adversarial Pairs.}
    \label{tab:pairs}
\end{table}
To make the task-oriented word semantics clearer, we show some examples of $(\epsilon,\delta)$-Collaborative Pairs and Adversarial Pairs found by our algorithm in Table \ref{tab:pairs}.
\alt{The $(\epsilon,\delta)$ is set to be $(0.012,0.1)$ for collaborative pairs, and $\alt{(0.2,0.01)}$ for adversarial pairs.
Note that the size of $(\epsilon,\delta)$ for adversarial pairs is significantly different from that for collaborative pairs. In collaborative pairs, $\epsilon$ is \xy{set to} a relatively small positive value, ensuring that the embedding distances in RNN are small, while for adversarial pairs, $\epsilon$ is set \xy{to a} larger value to avoid strict constraints on semantic distance that would make the resulting adversarial words too similar. The value of $\delta$ is also set for similar reasons.} From these examples, we identify
two intriguing properties of the source RNN.
The \textit{collaborative pairs} are the pair of words which the source RNN processes similarly with regard to the current task, but not synonyms in conventional semantics.
On the other hand, the \textit{adversarial pairs} are actually synonyms,
but when considered in the current task, the behaviours of RNNs are triggered 
% distinctly
differently. 
% on the pair. 
These contrastive pairs capture the RNN's specific understanding of word semantics, which are task-oriented.
% by the task. 
% These observations show the RNN's specific understanding of words, which is oriented by the task. 

Next, we analyze the adversarial pairs with a concrete example. 
Note that the collaborative pairs can be analyzed in a similar way.
Consider the adversarial pair (``exercise'', ``sports'') as an example, which are synonyms in general word semantics.
% But when we study their influence on the task, there would be a significant difference,
But when we analyze their influence on RNN's decisions,
% there are 
they demonstrate significant differences.
The influence analysis results show that
% since 
``exercise'' is a word that has high influence score on class ``Health'', 
while ``sports'' is a word that is most influential
% obviously correlated 
to the ``Sports'' category.
We further present an adversarial example generated by leveraging this adversarial pair. 
Consider the following sentence in the test set, 
% as input, 
``exercise helps her age swimmingly'', on which the RNN outputs ``Health'' with probability of 98.9\%. 
When we feed the sentence 
``sports helps her age swimmingly" to the RNN instead,
the output probability of category ``Sport''  raises up to 92.7\%.
However, the two sequences have nearly the same semantics.
% By all analysis above, as the synonyms and semantics are differently understood by RNNs and natural approaches, we can conclude that the semantics learnt by RNNs are task-oriented, which is different from the original meanings.
Based on the above result, we see that synonyms with regard to general embeddings are understood differently by RNNs.
Therefore, TME and TME-based explanation can help us better understanding what the target RNN learns and how it makes decisions.
% Further, we can conclude that the semantics learnt by RNNs are task-oriented, which is different from the original meanings.
%\xy{Consider whether it is spurious adversarial example, since WFA is just an abstraction, there may be inconsistency between WFA and RNN}

% This can help us better understanding what the target RNN learn and how it make decisions.
%This can partly explain how RNNs make decisions. 
% In brief, 
In this way,
by identifying and analyzing the collaborative examples, 
we can understand what are 
% classification-guided 
task-oriented
synonyms from the target RNN's perspective,
though they may be distinct in 
% natural
conventional embedding semantics.
On the other hand,
characterizing  adversarial pairs provides explanations of the target RNN on distinguishing similar words in the context of the current task.
% On the other hand,
% characterizing the adversarial pairs provides explanations on what the RNN distinguish the words.
We further validate the effectiveness of the contrastive pairs with the following two applications.
% through the 
% kinds of 
% pairs above with two applications in the following.

\subsubsection{TME for RNN Pretraining}
The identification of collaborative pairs reveals that 
TME is able to characterize
task-oriented semantics, compared with the conventional embedding method like Glove.
%We verify this with considering TME as pretrained word embedding for training a RNN of the same task.
% Therefore, since the collaborative pairs share closer distance in TME but are far away in Glove, 
% TME reveals what the RNN actually understand the words in the task.
We next show the effectiveness of TME in boosting RNN training when serving as pretrained embeddings.
% , which validates that TME stands more precise semantic of the classification task than general word embeddings.

\alt{In the experiment, we consider training RNNs on the QC news dataset and \alt{Toxic dataset} with three word embedding initialization strategies: (i) TME, (ii) Glove, and (iii) random initialization.
% The result is compared in 
Figure~\ref{fig:pretrain} shows the comparison results.
We can see that the initialization with TME outperforms Glove and random initialization on convergence speed in terms of loss and accuracy on the test set, which validates the effectiveness of TME in boosting RNN pretraining.

Note that there is a steep rise in accuracy observed during the training process.
In fact, for general NLP tasks, neural networks tend to experience a rapidly initial improvement in accuracy and then reach a plateau as training progresses. In our benchmarks tasks, the initialization of word embedding vectors has a significant impact on the network's ability to learn the correct patterns. It is only after the network learns the correct semantics from these embeddings, then the neural network can enter the phase of improvement in accuracy. As our pre-trained word embeddings are initialised with clearer semantics, the network is able to reach this phase of improvement at an earlier stage in the training process.}

\begin{figure}[t]
    \centering
    \begin{tabular}{cc}
        \includegraphics[width=0.45\textwidth]{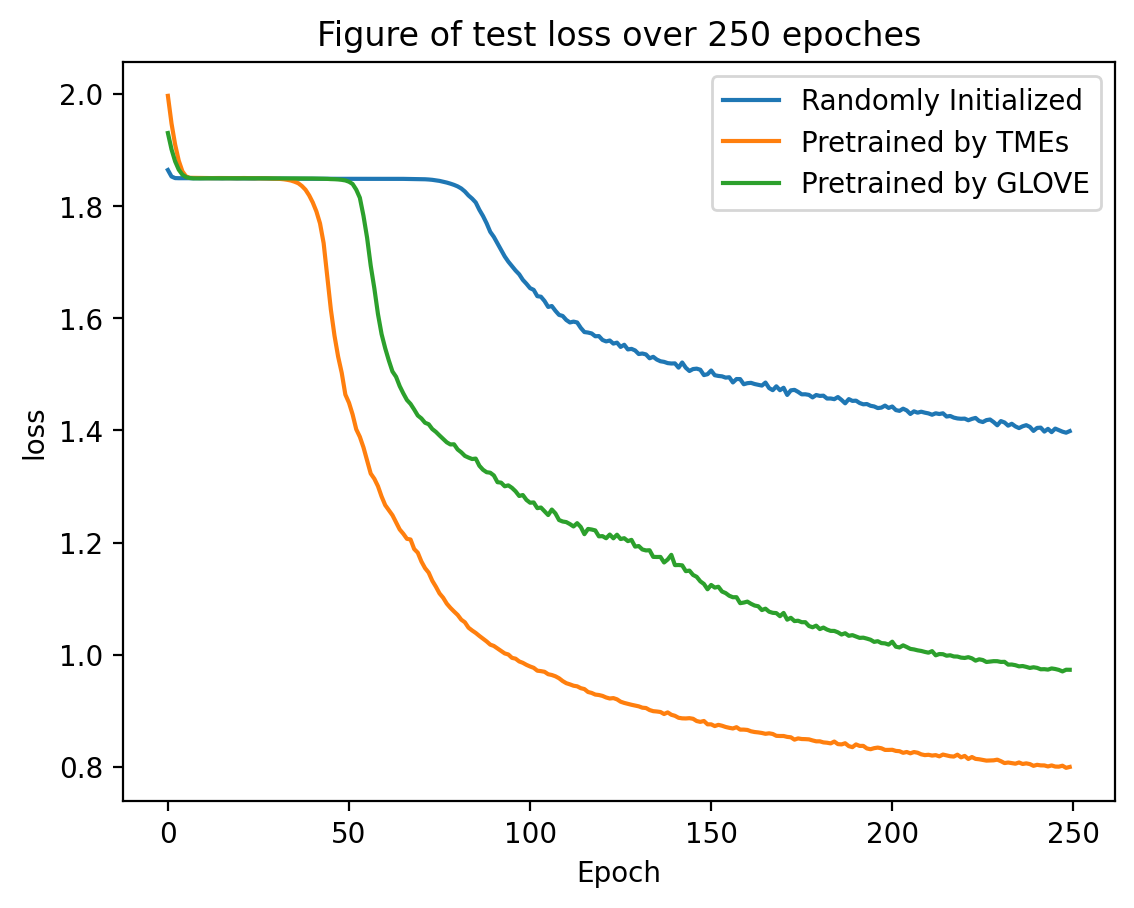} &        \includegraphics[width=0.45\textwidth]{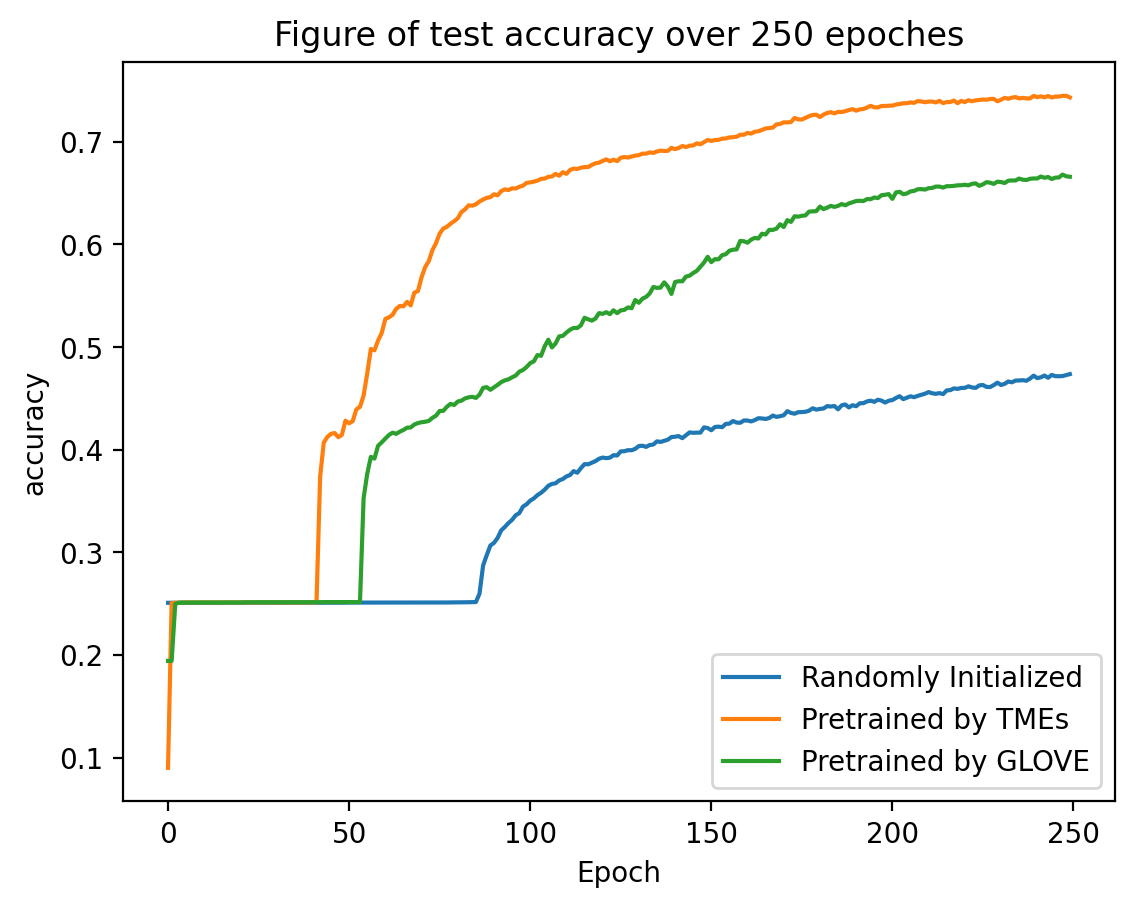}   \\
        (a) Loss on test set, QC & (b) Accuracy on test set, QC \\
        \includegraphics[width=0.45\textwidth]{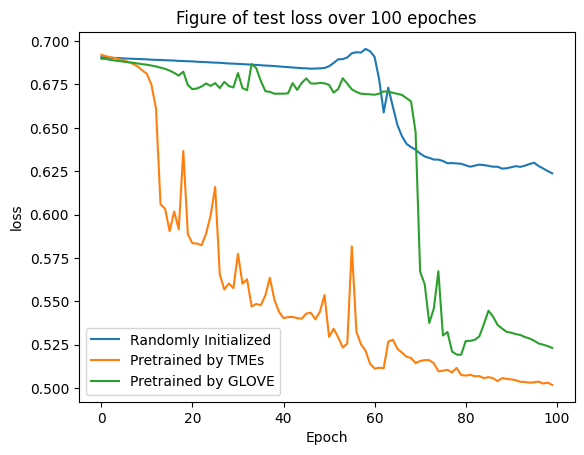} & \includegraphics[width=0.45\textwidth]{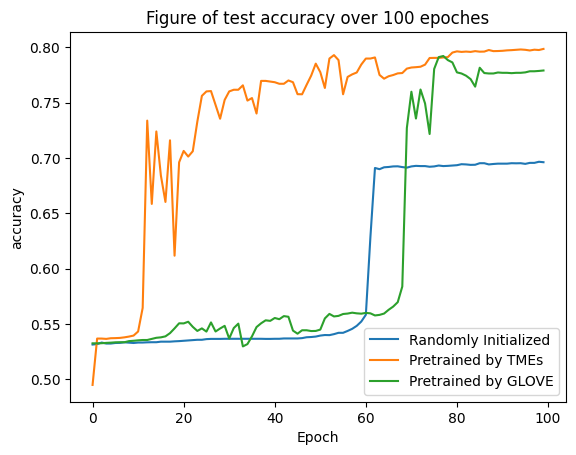} \\
        \alt{(c) Loss on test set, Toxic} & \alt{(d) Accuracy on test set, Toxic}
    \end{tabular}
        \caption{Comparison of three initialization strategies for RNN training.}
    \label{fig:pretrain}
\end{figure}

\subsubsection{TME for Adversarial Example Generation}
% We demonstrate the effectiveness of using TME in generating adversarial examples (AEs) for the source RNN, 
% which validates the ability of TME to accurately explain the decision logic and identify vulnerabilities of RNNs.
% The identification of adversarial pairs shows that TME are able to explain the decision logic and 
Previous investigation has shown that TME can be utilized to identify adversarial pairs and decision vulnerabilities of RNNs.
% based on TME can be utilized to generate adversarial examples for RNNs.
Inspired by the investigation results,
% these approache, 
we apply TME to generate adversarial examples for the source RNN.
We perform comparison experiments of using TME and Glove as embeddings in crafting adversarial examples.
% for the source RNN.
To evaluate the effectiveness of different methods in adversarial example generation, we use \textit{Attack Success Rate (ASR)} as the evaluation metric, namely
% that meaning a higher 
the proportion of crafted sequences that successfully mislead the RNN to produce false outputs.
% We show that TME outperform Glove with a higher Attack Success Rate (ASR), meaning a higher proportion of crafted sequences successfully mislead the RNN to produce false outputs.

To generate adversarial examples, we select the top-$k$ influential words in each sentence from the test set, measured by $\ell_2$-norm of IS,  and replace them with their synonyms with regard to different embedding methods.
% We set the selected synonym $\sigma'$ to have a TME semantic distance to the original word $\sigma$ with a certain lower bound, 
% that is $d_T(\sigma, \sigma')\ge 0.01$, to ensure the pair to be adversarial.
\alt{To ensure the generation of adversarial pair, we set a lower bound for the TME semantic distance between the original word $\sigma$ and the selected synonym $\sigma'$, that is, $d_T(\sigma, \sigma')\ge 0.01$.
For Glove, we simply replace the top-$k$ influential words with their synonyms.
% of the target words in 
% within the $d_S$-bound.
For TME, we leverage the Adversarial Pair under two settings with
% $\alt{(0.15, 0.01)}$, $\alt{(0.18, 0.01)}$\}
$d_S(\sigma, \sigma')\le 0.15$, $d_S(\sigma,\sigma') \le 0.18$, respectively,
to generate adversarial examples. Here, when $\epsilon$ is set to 0.15, the constraint on semantic distance for natural language is relatively strict. The resulting adversarial samples are semantically clear and have minor changes compared to the original sentences. We denote this kind of adversarial pair as \textit{Weak Adversarial Pairs}. 
%\xyc{Read carefully of this paragraph. And see whether there are errors in the presentation and smooth for better readability. Maybe split into two paragraphs.} 
When $\epsilon$ is set to 0.18, however, the constraint on semantic distance becomes more relaxed. \xy{The resulting sentences have more different semantics and there might be some local grammar issues, but still can be classified into the same label as the original ones.} 
We refer to this type of adversarial pair as \textit{Strong Adversarial Pairs}. 
For Toxic, we set $\epsilon$ to 0.12 and 0.2 for Weak Adversarial Pairs and Strong Adversarial Pairs, respectively.}
The comparison results are shown in Table~\ref{tab:asr}.
We can see that using adversarial pairs guided by TME achieves
% stands significantly 
higher ASR than using Glove. \alt{For example, on QC news dataset, we've gained an average increase of 
21\% for Weak Adversarial Pairs and
32\% for Strong Adversarial Pairs.}
The evaluation results validate the effectiveness of TME in capturing the  decision logic and vulnerability of the target RNN.

% \zmw{Todo: Merge the redundant part of Sec 7.1 and 7.2}
%\begin{table}[t]
    %\centering
    %\begin{tabular}{|c|c|c|}
     %\hline
      % Embeddings &  $k=1$ & $k=2$ \\
        %  \hline
      %  Glove &  0.17 & 0.22 \\
       %  \hline
     %   $(\alt{0.15, 0.01})$-Adversarial Pairs & 0.34 & 0.46 \\
    %    $(\alt{0.18,0.01})$-Adversarial Pairs & 0.44 & 0.59 \\
   %      \hline
  %  \end{tabular}
 %   \caption{Comparison results of ASR with different embedding methods.}
  %  \label{tab:asr}
%\end{table}

    \begin{table}
        \centering
        {
        \begin{tabular}{|c|c|c|c|c|}
            \hline
            \multirow{2}{*}{\alt{Embeddings}}& \multicolumn{2}{c|}{\alt{QC}} & \multicolumn{2}{c|}{\alt{Toxic}}\\
            \cline{2-5}
            &  \alt{$k=1$} & \alt{$k=2$}  & \alt{$k=1$} & \alt{$k=2$} \\
              \hline
            \multirow{1}{*}{Glove}  
            & \alt{0.17} & \alt{0.22} 
            & \alt{0.06} & \alt{0.11} \\
            \hline
            \multirow{1}{*}{\alt{Weak Adversarial Pairs}} 
            & \alt{0.34} & \alt{0.46} 
            & \alt{0.11} & \alt{0.23} \\
            \hline
            \multirow{1}{*}{\alt{Strong Adversarial Pairs}}  
            & \alt{\textbf{0.44}} & \alt{\textbf{0.59}} 
            & \alt{\textbf{0.15}} & \alt{\textbf{0.25}}\\
            \hline
            \end{tabular}}
            \caption{Comparison results of \alt{Attack Success Rate (ASR)} with different embedding methods. }
            \label{tab:asr}
    \end{table}

\subsection{\alt{Discussion}}

\paragraph{Computational Complexity}
    The time complexity of the whole workflow is analyzed as follows.
    Suppose that the set of training samples $\mathcal D_0$ 
    % has
    has
    $N$ words in total
    and its alphabet $\Sigma$ contains $n$ words, and is augmented as $\mathcal D$ with $t$ epochs (i.e. each sentence in $\mathcal{D}_0$ is transformed to $t$ new sentences in $\mathcal{D}$),
    hence $|\mathcal{D}|=(t+1)N$.
    Assume that a probabilistic output of RNNs is a $m$-dim vector,
    and the abstract states set $\hat S$ contains $k$ states.
    
    To start with, 
    the augmentation of $\mathcal D_0$ and tracking of probabilistic outputs in $\mathcal{D}$ will
    be completed in $\mathcal O(|\mathcal D|)=\mathcal O(t\cdot N)$ time.
    % Next, 
    Besides, the time complexity of k-means clustering algorithm is $\mathcal O(k\cdot |\mathcal D|)=\mathcal O(k\cdot t\cdot N)$.
    The count of abstract transitions will be done in $\mathcal O(n)$.
    As for the processing of transition matrices, we need
    to
    calculate the transition 
    % weight 
    probability for each 
    word $\sigma$ with each 
    % original 
    source state $\hat s_i$ and 
    % new 
    destination state $\hat s_j$,
    which costs $\mathcal O(k^2\cdot n)$ time. Finally, the context-aware enhancement on transition matrices takes $\mathcal{O}(k\cdot n)$ time.

    Note that $\mathcal O(n)=\mathcal O(N)$, hence we can conclude that the time complexity of our whole workflow
    is $\mathcal O(k^2\cdot t\cdot N)$.
    So the time complexity of our approaches only takes linear time w.r.t. the 
    % scale of 
    size of the dataset,
    which provides theoretical 
    % assurance 
    extraction overhead for
    large-scale data applications.

    \alt{
For the explanation analysis, the IS score can be computed in constant time, as this process simply involves multiplying two matrices. Assuming the vocabulary contains a total of $n$ words, and a sequence $s$ with $k$ words, conducting an adversarial attack on this sequence requires $\mathcal O(k + m\log k)$ time to find the top-$m$ influential words, and it costs $\mathcal{O}(n)$ time to find an optimal adversarial pair in the entire vocabulary through enumeration. Thus, if $m$ words need to be replaced, the time complexity for the entire process is $O(k + m\log k + nm)$.
}

\paragraph{\alt{Applicability to other RNNs}}

\alt{Although the proposed framework for WFA extraction and explanation of RNNs is customized for natural language tasks, we point out that some of its components can be generalized to other types of RNNs as well.}

\xy{
First, 
% the extraction algorithm in the proposed framework identifies 
the identified transition sparsity and context-awareness problems in WFA extraction for natural language tasks
% These issues 
may also occur in RNNs used in other domains,
thus the proposed methods to address these problems are applicable to them as well.
Thus, the empirical method to complement the missing rules in the transition diagram and the adjustment of transition matrices to enhance the context-awareness of the WFA can also be applied to other types of RNNs.
However, the data augmentation tactics proposed in the paper may need to be adapted to suit the specific characteristics of other types of RNNs. 
Specifically, we can perform data augmentation on natural language samples as long as the synthetic sentences make sense. However, other datasets, such as formal languages, do not possess this property.}

\alt{As for the explanation analysis, 
% we would like to emphasize that 
the application of our method for RNN explanation is not limited to natural language tasks. 
As long as a WFA can be extracted from the target RNN, the method for explanation is applicable. 
In fact, our study on RNN explainability only involves the extraction of vector representation of words % \xyc{word2vec is a classic algorithm in NLP; change into another expression, e.g., vector representation of words} 
through the transition matrices of the WFA,
%enabling the embedding of samples into linear space.
thus this \xy{component of our framework} is highly generalizable and can be applied to various other domains beyond natural language processing.} 
%With this in mind, it is promising to see the potential of this framework in bridging the gap between deep learning and interpretable machine learning.
\section{Related Work} 
\label{sec:related}
Many research efforts have been made to understand the behaviours of RNNs with an extracted model.
We discuss the extraction methods and their applications respectively in the following.

\subsection{Model Extraction of RNNs}
As Jacobsson reviewed in~\cite{jacobsson2005},
the extraction approach of RNNs can be divided into two categories: pedagogical approaches and compositional approaches.

\paragraph*{Pedagogical Approaches}
% Much progress has been achieved 
Many research works consider using pedagogical approaches 
to abstract RNNs by leveraging explicit learning algorithms,
such as the $L^*$ algorithm~\cite{angluin}.
% in 1987.
Earlier works date back to two decades ago, when Omlin et al. attempted to extract
a finite model for Boolean-output RNNs~\cite{omlin1992,omlin1996,omlin1996b}.
% \xy{Is this ``succedaneum'' a proper concept?}
% \zm{maybe ``model`` is enough}
Recently, Weiss et al.~\cite{weiss2018} proposed an approach to extracting DFA from RNN-acceptors based on $L^*$ algorithm.
Later, they 
presented a weighted extension of $L^*$ algorithm that extracted probabilistic determininstic finite automata (PDFA)
from RNNs~\cite{weiss2019}.
Okudono et al.~\cite{ok2020} proposed a weighted extension of $L^*$ algorithm to extract WFA for real-value-output RNNs.
Overall, pedagogical approaches have achieved great success in abstracting RNNs for small-scale languages,
particularly formal languages. 
% On the other hand, 
% these approaches are limited by the scale of the language model,
Such exact learning approaches have intrinsic limitation in 
% are limited in 
scalability w.r.t. the language complexity,
% hence they
thus are not suitable for automata extraction for natural language models.

\paragraph*{Compositional Approach}
Another technical line for automata extraction from RNNs
% way 
is the compositional approach, 
% specifically 
which leverages
% using 
unsupervised algorithms (e.g. k-means, GMM) to cluster state vectors as abstract states~\cite{zeng,cechin}.
Wang et al.~\cite{wang2018} studied the key factors in the compositional approach that
influence the reliability of the extraction process, and proposed an empirical rule to extract DFA from RNNs.
Later, Zhang et al.~\cite{zhang2021} followed the state encoding of compositional approach and proposed a WFA extraction approach
from RNNs,
which can be applied to both grammatical languages and natural languages.
In this paper, the proposed WFA extraction approach
from RNNs also falls into the line of compositional approach,
but aims at proposing transition rule extraction method to address the transition sparsity problem 
and enhance the context-aware ability, which is customized for natural language tasks.

% \subsection{Analysis and Explanation of RNNs}
\subsection{Model-based RNN Analysis  and Explanation}
% Another line of 
There are a series of  works focusing on 
% using 
deriving the extracted models for
further applications, where the abstract models are more amenable to analysis and explanation.
% % to analyze and explain the target RNNs.
% We divide them into two categories: analysis of RNNs and explanation of RNNs.
%In the past few years, much progress has been made for RNNs in both theory and implementation. Along with the growth in model and algorithmic complexity, it also poses more challenges to the interpretability problem of RNN models. To this end, many research have emphasized on the interpretability of RNN models~\cite{hou2020learning, freitas2014comprehensible}. We introduce some significant research that has made great effort on this field, which are similar to the research direction and interest of our article. Those research are done in several approches, which can be divided into two categories: modeling approaches and algorithm approaches.

%\paragraph*{Modeling Approaches} 
\paragraph*{Model-based Analysis}
Model extraction techniques have been widely used to aid the analysis of RNNs,
% This thread of works focus on analysis the target RNNs based on the extracted surrogate models.
since the extracted models can be regarded as an approximation of the target RNNs, on which are easier to operate and perform analysis.
% we can analysis their properties
% with the aid of the surrogate model.
% \textit{Deepsteller}~
\cite{du2019} is a representative work for model-based RNN analysis, which leverages the extracted model
% provides various criteria 
to detect adversarial examples and increase test coverage of the target RNNs.
% based on the extracted model.
Later, 
% Du et al. proposed \textit{Marble}~
\cite{du2020} proposed a 
% , which is a 
model-based approach for robustness analysis of RNNs. 
% stateful deep learning systems.
% Besides, 
Xie et al.~\cite{xie2021} proposed to leverage the extracted model to identify buggy behaviors and further for automatic repairment of 
% the incorrect behaviours of 
RNNs.
% Automatically.
 % proposed \textit{RNNRepair}~
In this paper, based on the extracted WFA, 
we proposed a new embedding method TME, which provides a new insight on RNN analysis for natural language tasks.
With the proposed contrastive pairs derived by TME,
% and general word embeddings,
we can analyze task-oriented semantics of the target RNNs,
% learns from the task.
% Our TME can also 
which further can be applied to
boost pretraining and adversarial example generation for RNNs.

% verifying, analyzing and repairing RNNs
% by leveraging these techniques. 

% \paragraph*{Explanation Algorithms of RNNs}
\paragraph*{Model-based Explanation}
There are also several works devoted to explaining the mechanism of RNNs with the aid of surrogate models. 
Krakovna et al.~\cite{krakovna2016increasing} presented an interpretation method for RNNs by using hidden markov models (HMMs) to simulate the source RNNs.
Hou et al.~\cite{hou2020learning} proposed an approach to interpreting the effect of gates  on the mechanism of RNNs by using the extracted finite state automata.
Jiang et al.~\cite{jiang2020cold} proposed a hybrid model \textit{FA-RNNs}, which is trainable, generalizable as well as interpretable. 
There are also works operating directly on the structure of RNNs. 
Guo et al.~\cite{guo2018interpretable} proposed an interpretable LSTM neural network equipped with tensorized hidden states, which could learn variable-specific representations. 
% \xy{what is variable specific?}
%Motivated by the SISTA algorithm~\cite{gregor2010learning}, Wisdom et al.~\cite{wisdom2016interpretable} proposed \textit{SISTA-RNNs} as a form of interpretable recurrent neural networks based on a specific probabilistic model. 
% \xy{What is this specific prob model?}
In this work, by leveraging the extracted WFA, we proposed a global explanation method, which computes the  word-wise influence score on RNN decisions,
and a contrastive explanation method, where the identified collaborative and adversarial repairs effectively characterize the task-oriented semantics learned by the target RNN.

\section{Conclusion}
\label{sec:conclusion}
In this paper, we propose a general framework for weighted automata extraction and explanation of RNNs for natural language tasks.
We introduce
a novel approach to  extracting transition rules of
weighted finite automata from recurrent neural networks.
In particular, we address the transition sparsity problem and complement the transition rules of missing rows, which effectively improves the extraction precision.
In addition, we present an {heuristic}
method to enhance the context-aware ability of the extracted WFA.
We further propose two augmentation tactics to track more transition behaviours of RNNs.
Both theoretical analysis and experimental results demonstrate
the efficiency and precision of our rule extraction approach
for natural language tasks.
Based on the extracted model, we propose a word embedding method, Transition Matrix Embeddings (TME),
 to construct task-oriented explanations of the target RNN, including 
a word-wise global explanation method of RNNs, and a contrastive method to interpret the word semantics that the RNNs learned from the task.
\subsection*{Acknowledgements}
This work was sponsored by the National Natural Science Foundation of China under Grant No. 62172019, 61772038,
and CCF-Huawei Formal Verification Innovation Research Plan; partially funded by the ERC under the European Union’s Horizon 2020 research and innovation programme (FUN2MODEL, grant agreement No.~834115).
\bibliographystyle{elsarticle-num}
\bibliography{reference.bib}
\end{document}